\pdfoutput=1

\documentclass[11pt]{article}

\usepackage[]{ACL2023}

\usepackage{times}
\usepackage{latexsym}
\usepackage{comment}

\usepackage[T1]{fontenc}    

\usepackage{microtype}

\usepackage{inconsolata}
\usepackage{amsfonts}       
\usepackage{nicefrac}       
\usepackage{microtype}      
\usepackage{lipsum}		
\usepackage{graphicx}
\usepackage{natbib}
\usepackage{doi}

\usepackage{microtype}
\usepackage{subfigure}

\let\normallog \log
\usepackage{amsmath}
\DeclareMathOperator*{\argmax}{argmax}

\let\log \normallog
\usepackage{amssymb}
\usepackage{mathtools}
\usepackage{amsthm}
\usepackage{thm-restate}
\usepackage{xcolor}
\usepackage{xfrac}
\theoremstyle{plain}
\newtheorem{theorem}{Theorem}[section]

\newtheorem{definition}[theorem]{Definition}

\theoremstyle{remark}

\newcommand{\orangetext}[1]{{\color{orange} #1}}
\newcommand{\redtext}[1]{{\color{red} #1}}
\newcommand{\bluetext}[1]{{\color{blue} #1}}

\newcommand{\vtheta}{{ \boldsymbol \theta}}
\newcommand{\vphi}{{ \boldsymbol \phi}}
\newcommand{\vTheta}{\boldsymbol{\Theta}}
\newcommand{\vThetastar}{\boldsymbol{\Theta}^{\star}}

\newcommand{\q}{q}

\newcommand{\Rd}{\mathbb{R}^D}

\newcommand{\hrlace}{h_{\textsc{r}}}

\newcommand{\CLS}{ \texttt{[CLS]}}

\newcommand{\defeq}[0]{\mathrel{\stackrel{\textnormal{\tiny def}}{=}}}
\usepackage[capitalize]{cleveref}
\crefname{section}{\S}{\S\S}
\Crefname{section}{\S}{\S\S}
\crefname{table}{Tab.}{}
\crefname{figure}{Fig.}{}
\crefname{proposition}{Prop.}{}
\crefname{algorithm}{Alg.}{}
\crefname{appendix}{App.}{}
\crefname{lemma}{Lemma}{}
\Crefname{theorem}{Theorem}{}
\crefname{prop}{Proposition}{}
\crefname{claim}{Claim}{}
\crefname{myexample}{Example}{}
\crefname{align}{}{}
\crefname{equation}{Eq.}{eq.}

\newcommand{\R}{\mathbb{R}}
\newcommand{\xx}{\boldsymbol{x}}
\newcommand{\aalpha}{\boldsymbol{\alpha}}

\newcommand{\ttheta}{\boldsymbol{\theta}}

\newcommand{\xxn}{\xx_n}

\newcommand{\entropy}{\mathrm{H}}

\newcommand{\defn}[1]{\textbf{#1}}
\newcommand{\yhat}{\widehat{y}}

\newcommand{\zn}{z_n}
\newcommand{\zzn}{z_n}

\newcommand{\valpha}{\boldsymbol{\valpha}}

\newcommand{\rhodelta}{\rho_{\delta}}
\newcommand{\CAv}{\mathrm{A}_{\Vfam}}

\usepackage{booktabs}
\usepackage{multirow}

    \theoremstyle{plain}
    \newtheoremstyle{TheoremNum}
        {\topsep}{\topsep}              
        {\itshape}                      
        {}                              
        {\bfseries}                     
        {.}                             
        { }                             
        {\thmname{#1}\thmnote{ \bfseries #3}}
    \theoremstyle{TheoremNum}

\newcommand{\rv}[1]{{\color{black} \mathrm{#1}}}
\renewcommand{\R}{{\color{black} \mathbb{R}}}
\newcommand{\X}{{\boldsymbol{ \rv{X}}}}
\newcommand{\Y}{{\rv{Y}}}
\newcommand{\Yadv}{{\YHAT_{\mathrm{a}}}}
\newcommand{\Yrnd}{{\YHAT_{\mathrm{p}}}}

\newcommand{\Z}{{\rv{Z}}}
\newcommand{\ZHAT}{{\rv{\widehat{Z}}}}

\newcommand{\ZHATq}{{\rv{\widehat{Z}}}_q}

\newcommand{\ZHATr}{{\rv{\widehat{Z}}}_r}
\newcommand{\mi}{{\color{black} \mathrm{I}}}

\newcommand{\indk}{{\color{black} \iota_k}}

\newcommand{\gapindh}{{\color{black} \text{GAP}_{\text{ind}}(\rv{\widehat{Y}} \to \rv{Z} \mid h(\X))}}

\newcommand{\Iacc}{\mathrm{I}^{\mathrm{A}}_{\Vfam}}

\newcommand{\indki}{{\color{black} \iota_k(i)}}
\newcommand{\indjk}{{\color{black} \iota_j(k)}}
\newcommand{\indkm}{{\color{black} \iota_m(k)}}
\newcommand{\indkj}{{\color{black} \iota_j(k)}}

\newcommand{\Vfam}{\color{black} \mathcal{V}}
\newcommand{\Vfamdelta}{\color{black} \mathcal{V}^\delta}

\newcommand{\YHAT}{{ \rv{\widehat{Y}}}}

\newcommand{\sizey}{ \color{black} |\mathcal{Y}|}

\newcommand{\CHvdelta}[1]{{\color{black} \mathrm{H}_{\mathcal{V}^{\delta}}(#1)}}

\newcommand{\CHv}[1]{{\color{black} \mathrm{H}_{\mathcal{V}}(#1)}}
\newcommand{\E}[1]{\color{black} \mathop{\mathbb{E}}_{#1}}

\newcommand{\CHvuncond}[1]{{\color{black} \mathrm{H}_{\mathcal{V}}\left(#1\right)}}

\newcommand{\Zemp}{\widetilde{\Z}} 
\newcommand{\Xemp}{\widetilde{\X}}

\newcommand{\Ivxz}{\color{black} \mi_{\mathcal{V}}(\X \to \Z)}

\newcommand{\Ivhxz}{\color{black} \mi_{\mathcal{V}}(h(\X) \to \Z)}
\newcommand{\Ivdeltahxz}{\color{black} \mi_{\mathcal{V}^{\delta}}(h(\X) \to \Z)}
\newcommand{\Ivxzemp}{\color{black} \mi_{\mathcal{V}}(h(\Xemp) \to \Zemp)}
\newcommand{\IvxzNoh}{\color{black} \mi_{\mathcal{V}}(\X \to \Z)}

\newcommand{\Ivyz}{\color{black}
\mi_{\mathcal{V}}( \mathrm{\widehat{Y}} \to \Z)}
\newcommand{\Ivdeltayz}{\color{black}
\mi_{\mathcal{V}^\delta}( \mathrm{\widehat{Y}} \to \Z)}

\newcommand{\dataset}{\color{black} {\mathcal{D}}}

\newcommand{\yset}{\mathcal{Y}}
\newcommand{\zset}{\mathcal{Z}}

\newcommand{\Yhatprof}{\Yrnd}
\newcommand{\Yhatadv}{\Yadv}

\newcommand{\Ivyzadv}{\color{black} \mi_{\mathcal{V}}(\color{black} \Yhatadv \to \Z)}
\newcommand{\Ivyzprof}{\color{black} \mi_{\mathcal{V}}( \Yhatprof \to \Z)}

\newcommand{\loss}{\color{black} \ell}
\newcommand{\Vdelta}{\color{black} \Vfam^{\delta}}

\newcommand{\softmax}{\color{black} \mathrm{softmax}}

\title{Log-linear Guardedness and its Implications}

 \author{Shauli Ravfogel\textsuperscript{\normalfont1,2} \,  Yoav Goldberg\textsuperscript{\normalfont1,2} \,Ryan Cotterell\textsuperscript{\normalfont 3}\\
\textsuperscript{1}Bar-Ilan University \, \textsuperscript{2}Allen Institute for Artificial Intelligence \,
\textsuperscript{3}ETH Zürich \\ 
  {\tt\{\href{mailto:shauli.ravfogel@gmail.com}{shauli.ravfogel}, \href{mailto:yoav.goldberg@gmail.com}{yoav.goldberg}\}@gmail.com}\,
\tt{\href{mailto:ryan.cotterell@inf.ethz.ch}{ryan.cotterell@inf.ethz.ch}}
   }

\usepackage{tikz,siunitx}
\usetikzlibrary{shapes.geometric,shapes.symbols}

\usepackage{todonotes}
\makeatletter
\newcommand*\iftodonotes{\if@todonotes@disabled\expandafter\@secondoftwo\else\expandafter\@firstoftwo\fi} 
\makeatother

\begin{document}
\maketitle

\begin{abstract}
Methods for erasing human-interpretable concepts from neural representations that assume linearity have been found to be tractable and useful.
However, the impact of this removal on the behavior of downstream classifiers trained on the modified representations is not fully understood.
In this work, we formally define the notion of log-linear guardedness as the inability of an adversary to predict the concept directly from the representation, and study its implications.
We show that, in the binary case, under certain assumptions, a downstream log-linear model cannot recover the erased concept.
However, we demonstrate that a multiclass log-linear model \emph{can} be constructed that indirectly recovers the concept in some cases, pointing to the inherent limitations of log-linear guardedness as a downstream bias mitigation technique.
These findings shed light on the theoretical limitations of linear erasure methods and highlight the need for further research on the connections between intrinsic and extrinsic bias in neural models.
\newline
\newline
\vspace{1.0em} 
\hspace{.5em}\includegraphics[width=1.25em,height=1.25em]{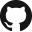}\hspace{.5em}\parbox{\dimexpr\linewidth-3\fboxsep-3\fboxrule}{\url{https://github.com/rycolab/guardedness}}
\end{abstract}

\section{Introduction}
Neural models of text have been shown to represent human-interpretable concepts, e.g., those related to the linguistic notion of morphology \citep{vylomova2017word}, syntax \citep{linzen2016assessing}, semantics \citep{belinkov2017evaluating}, as well as extra-linguistic notions, e.g., gender distinctions \citep{caliskan2017semantics}. 
Identifying and erasing such concepts from neural representations is known as \defn{concept erasure}.
Linear concept erasure in particular has gained popularity due to its potential for obtaining formal guarantees and its empirical effectiveness \cite{bolukbasi2016man, dev2019attenuating, ravfogel-etal-2020-null, dev2020oscar, kaneko2021debiasing, shao-etal-2023-gold, shao2023erasure, kleindessner2023efficient,belrose2023leace}.\looseness=-1

A common instantiation of concept erasure is removing a concept (e.g., gender) from a representation (e.g., the last hidden representation of a transformer-based language model) such that it cannot be predicted by a log-linear model.
Then, one fits a \emph{secondary} log-linear model for a downstream task over the erased representations. 
For example, one may fit a log-linear sentiment analyzer to predict sentiment from gender-erased representations. 
The hope behind such a pipeline is that, because the concept of gender was erased from the representations, the predictions made by the log-linear sentiment analyzer are oblivious to gender.
Previous work \citep{ravfogel-etal-2020-null, elazar2021amnesic, jacovi2021contrastive, ravfogel2022linear} has implicitly or explicitly relied on this assumption that erasing concepts from representations would also result in a downstream classifier that was oblivious to the target concept.\looseness=-1

In this paper, we formally analyze the effect concept erasure has on a downstream classifier.
We start by formalizing concept erasure using \citeposs{xu2020theory} $\mathcal{V}$-information.\footnote{We also consider a definition based on accuracy.} 
We then spell out the related notion of \defn{guardedness} as the inability to predict a given concept from concept-erased representations using a specific family of classifiers.
Formally, if $\Vfam$ is the family of distributions realizable by a log-linear model, then we say that the representations are guarded against gender with respect to $\Vfam$.
The theoretical treatment in our paper specifically focuses on \defn{log-linear guardedness}, which we take to mean the inability of a \emph{log-linear} model to recover the erased concept from the representations.
We are able to prove that when the downstream classifier is binary valued, such as a binary sentiment classifier, its prediction indeed cannot leak information about the erased concept (\cref{sec:binary}) under certain assumptions. 
On the contrary, in the case of multiclass classification with a log-linear model, we show that predictions \emph{can} potentially leak a substantial amount of information about the erased concept, thereby recovering the guarded information completely.

The theoretical analysis is supported by experiments on commonly used linear erasure techniques (\cref{sec:experiments}). 
While previous authors (\citealt{goldfarb2021intrinsic}, \citealt{orgad2022gender}, \textit{inter alia}) have empirically studied concept erasure's effect on downstream classifiers, to the best of our knowledge, we are the first to study it theoretically.
Taken together, these findings suggest that log-linear guardedness may have limitations when it comes to preventing information leakage about concepts and should be assessed with extreme care, even when the downstream classifier is merely a log-linear model.\looseness=-1

\section{Information-Theoretic Guardedness}\label{sec:info-theory}
In this section, we present an information-theoretic approach to guardedness,
which we couch in terms of  $\mathcal{V}$-information \cite{xu2020theory}.

\subsection{Preliminaries}
We first explain the concept erasure paradigm \cite{ravfogel2022linear}, upon which our work is based.
Let $\X$ be a representation-valued random variable.
In our setup, we assume representations are real-valued, i.e., they live in $\R^D$.
Next, let $\Z$ be a binary-valued random variable that denotes a protected attribute, e.g., binary gender.\footnote{Not all concepts are binary, but our analysis in \Cref{sec:info-theory} makes use of this simplifying assumption.} 
We denote the two binary values of $\Z$ by $\zset \defeq \{\bot, \top\}$.
We assume the existence of a \defn{guarding function} $h : \R^D \rightarrow \R^D$ that, when applied to the representations, removes the ability to predict a concept $\Z$ given concept by a specific family of models. 
Furthermore, we define the random variable $\YHAT = t(h(\X))$ where $t : \R^D \rightarrow \yset \defeq \{0, \dots, \sizey\}$ is a function\footnote{The elements of $\yset$ are denoted $y$.} that corresponds to a linear classifier for a downstream task.
For instance, $t$ may correspond to a linear classifier that predicts the sentiment of a representation.
\looseness=-1

Our discussion in this paper focuses on the case when the function $t$ is derived from the argmax of a log-linear model, i.e., in the binary case we define $\YHAT$'s conditional distribution given $h(\X)$ as
\begin{equation}\label{eq:binary-yhat}
    p(\YHAT = y \mid h(\X) = h(\xx)) = 
\begin{cases}

    1, &\textbf{if } y = y^\ast \\
    
    0, &\textbf{else}
\end{cases} 
\end{equation}
where $\vtheta \in \R^D$ is a parameter column vector, $\phi \in \R$ is a scalar bias term, and
\begin{equation}
y^\ast = \begin{cases} 1, \quad  \textbf{if}\quad \vtheta^{\top} h(\xx) + \phi > 0 \\
0, \quad  \textbf{else}\end{cases}
\end{equation}
And, in the multivariate case we define $\YHAT$'s conditional distribution given $h(\X)$ as
\begin{equation}\label{eq:multi-yhat}
    p(\YHAT = y \mid h(\X) = h(\xx)) = 
\begin{cases}

    1, &\textbf{if } y = y^\ast \\
    
    0, &\textbf{else}
\end{cases} 
\end{equation}
where $y^\ast = \argmax_{y' \in \yset}\,\, (\vTheta^{\top}h(\xx)  + \vphi)_{y'}$
and $\vTheta_{y} \in \R^D$ denotes the $y^{\text{th}}$ column of $\vTheta \in \R^{D \times K}$, a parameter matrix, and $\vphi \in \R^K$ is the bias term.
Note $K$ is the number of classes.\looseness=-1

\subsection{$\mathcal{V}$-Information}

Intuitively, a set of representations is guarded if it is not possible to predict a protected attribute $z \in \zset$ from a representation $\xx \in \R^D$ using a specific predictive family.
As a first attempt, we naturally formalize predictability in terms of mutual information.
In this case, we say that $\Z$ is not predictable from $\X$ if and only if $\mi(\X; \Z) = 0$.
However, the focus of this paper is on \emph{linear} guardedness, and, thus,
we need a weaker condition than simply having the mutual information $\mi(\X; \Z) = 0$.
We fall back on \citeposs{xu2020theory} framework of $\mathcal{V}$-information, which introduces a generalized version of mutual information.
In their framework, they restrict the predictor to a family of functions $\Vfam$, e.g., the set of all log-linear models.\looseness=-1

We now develop the information-theoretic background to discuss $\mathcal{V}$-information. 
The \defn{entropy} of a random variable is defined as
\begin{equation}
   \entropy(\Z) \defeq - \E{z\sim p(\Z)} {\log p(z)}
\end{equation}
\citet{xu2020theory} analogously define the \defn{conditional $\Vfam\text{-entropy}$} as follows
\begin{equation}\label{eq:v-conditional-entropy}
    \CHv{\Z \mid \X} \defeq -\sup_{q \in \Vfam} \E{(\xx,z) \sim p(\X, \Z)} {\log q(z \mid \xx)}
\end{equation}
The \defn{$\Vfam\text{-entropy}$} is a special case of \Cref{eq:v-conditional-entropy} without conditioning on another random variable, i.e.,\looseness=-1
\begin{equation}
    \CHvuncond{\Z} \defeq -\sup_{q \in \Vfam} \E{z \sim p(\Z)} {\log q(z)}
\end{equation}
\citet{xu2020theory} further define the \defn{$\Vfam$-information}, a generalization of mutual information, as follows
\begin{equation}
    \IvxzNoh{} \defeq \CHvuncond{\Z} -  \CHv{\Z \mid \X}
    \label{def:vinfo}
\end{equation}
In words, \cref{def:vinfo} is the \emph{best} approximation of the mutual information realizable by a classifier belonging to the predictive family $\Vfam$. 
Furthermore, in the case of log-linear models, \Cref{def:vinfo} can be approximated empirically by calculating the negative log-likelihood loss of the classifier on a given set of examples, as $\CHvuncond{\Z}$ is the entropy of the label distribution, and $\CHv{\Z \mid \X}$ is the minimum achievable value of the cross-entropy loss.\looseness=-1

\subsection{Guardedness} 
Having defined $\Vfam$-information, we can now formally define guardedness as the condition where the $\Vfam$-information is small. 

\begin{definition}[$\Vfam$-Guardedness]
Let $\X$ be a representation-valued random variable and let $\Z$ be an attribute-valued random variable.
Moreover, let $\Vfam$ be a predictive family.
A guarding function $h$ $\varepsilon$-\defn{guards} $\X$ with respect to $\Z$ over $\Vfam$ if $\Ivhxz < \varepsilon$.\looseness=-1
\end{definition}

\begin{definition}[Empirical $\Vfam$-Guardedness]
Let $\dataset = \{ (\xxn, \zzn) \}_{n=1}^N$ where $(\xxn, \zzn) \sim p(\X, \Z)$.
Let $\Xemp$ and $\Zemp$ be random variables over $\R^D$ and $\zset$, respectively, whose distribution corresponds to the marginals of the empirical distribution over $\dataset$. 
We say that a function $h(\cdot)$ \defn{empirically} $\varepsilon$-\defn{guards} $\dataset$ with respect to the family $\Vfam$ if $\Ivxzemp < \varepsilon$.

\label{def:information-theoretic-guardedness}
\end{definition}

In words, according to \cref{def:information-theoretic-guardedness}, a dataset is log-linearly guarded if no linear classifier can perform better than the trivial classifier that completely ignores $\rv{X}$ and always predicts $\Z$ according to the proportions of each label. The commonly used algorithms that have been proposed for linear subspace erasure can be seen as approximating the condition we call log-linear guardedness \citep{ravfogel-etal-2020-null, ravfogel2022linear, ravfogel2022adversarial}. 
Our experimental results focus on empirical guardedness, which pertains to practically measuring guardedness on a finite dataset. However, determining the precise bounds and guarantees of empirical guardedness is left as an avenue for future research.

\section{Theoretical Analysis}
In the following sections, we study the implications of guardedness on \emph{subsequent} linear classifiers.
Specifically, if we construct a third random variable $\YHAT = t(h(\rv{X}))$ where $t : \Rd \rightarrow \yset$ is a function, what is the degree to which $\YHAT$ can reveal information about $\Z$?
As a practical instance of this problem, suppose we impose $\varepsilon$-guardedness on the last hidden representations of a transformer model, i.e., $\X$ in our formulation, and then fit a linear classifier $t$ over the guarded representations $h(\X)$ to predict sentiment. 
Can the predictions of the sentiment classifier indirectly leak information on gender? 
For expressive $\mathcal{V}$, the data-processing inequality \citep[\S2.8]{cover1991elements} applied to the Markov chain $\X \to \YHAT \to \Z$
tells us the answer is no.
The reason is that, in this case, $\mathcal{V}$-information is equivalent to mutual information and the data processing inequality tells us such leakage is \emph{not} possible.
However, the data processing inequality does not generally apply to $\mathcal{V}$-information \cite{xu2020theory}.
Thus, it \emph{is} possible to find such a predictor $t$ for less expressive $\mathcal{V}$.
Surprisingly, when $|\mathcal{Y}| = 2$, we are able to prove that constructing such a $t$ that leaks information is impossible under a certain restriction on the family of log-linear models.\looseness=-1

\subsection{Problem Formulation}
We first consider the case where $|\yset| = 2$.
\subsection{A Binary Downstream Classifier}
\label{sec:binary}
We begin by asking whether the predictions of a binary log-linear model trained over the guarded set of representations can leak information on the protected attribute.
Our analysis relies on the following simplified family of log-linear models.
\begin{definition}[Discretized Log-Linear Models]
The family of \defn{discretized binary log-linear models} with parameter $\delta \in (0,1)$ is defined as 
\begin{equation}
\Vdelta \defeq \left\{f \Bigm| 
\begin{array}{l}
    f(0) = 
\rhodelta(\sigma(\aalpha^{\top} \xx + \gamma))\\
f(1) = \rhodelta(1 - \sigma(\aalpha^{\top} \xx + \gamma))
\end{array} \!\! \right\}
\end{equation}
with $\aalpha \in \R^D$, $\gamma \in \R$, $\sigma$ being the logistic function, and where we define the $\delta$-discretization function as\looseness=-1
\begin{equation}
     \rhodelta(p) \defeq
\begin{cases}
  \delta ,& \textbf{if } p \geq \frac{1}{2} \\
    1 - \delta,              & \textbf{else}
    \end{cases} 
\label{eq:discrete-logistic}
\end{equation}
In words, $\rhodelta$ is a function that maps the probability value to one of two possible values.
Note that ties must be broken arbitrarily in the case that $p=\frac{1}{2}$ to ensure a valid probability distribution.\looseness=-1
\label{def:discrete-logistic}
\end{definition}

\begin{figure*}
  \centering
  \subfigure[Log-linearly guarded data in $\R^2$ with axis-aligned clusters.] {\includegraphics[scale=0.25]{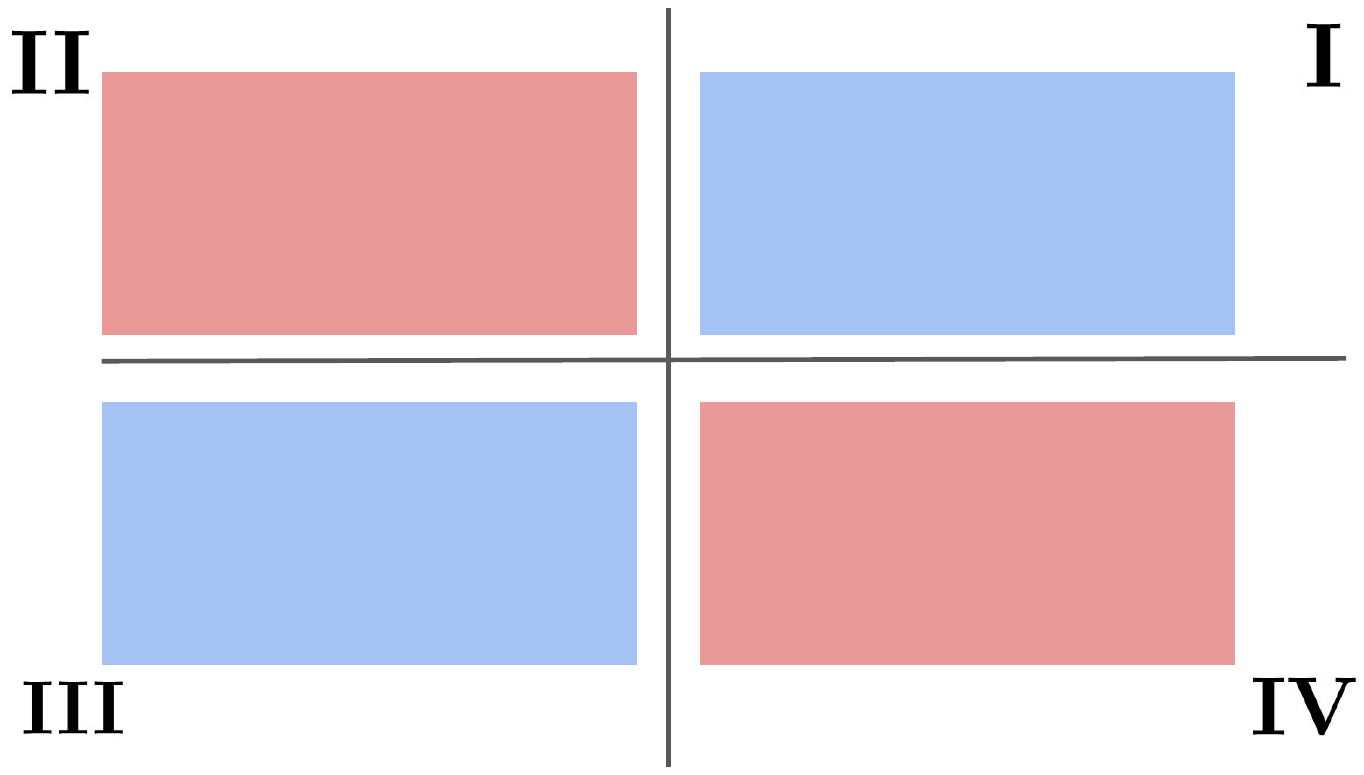} \label{fig:sep}}\quad
  \subfigure[Log-inearly guarded data in $\R^2$ with clusters that are not axis-aligned.] {\includegraphics[scale=0.25]{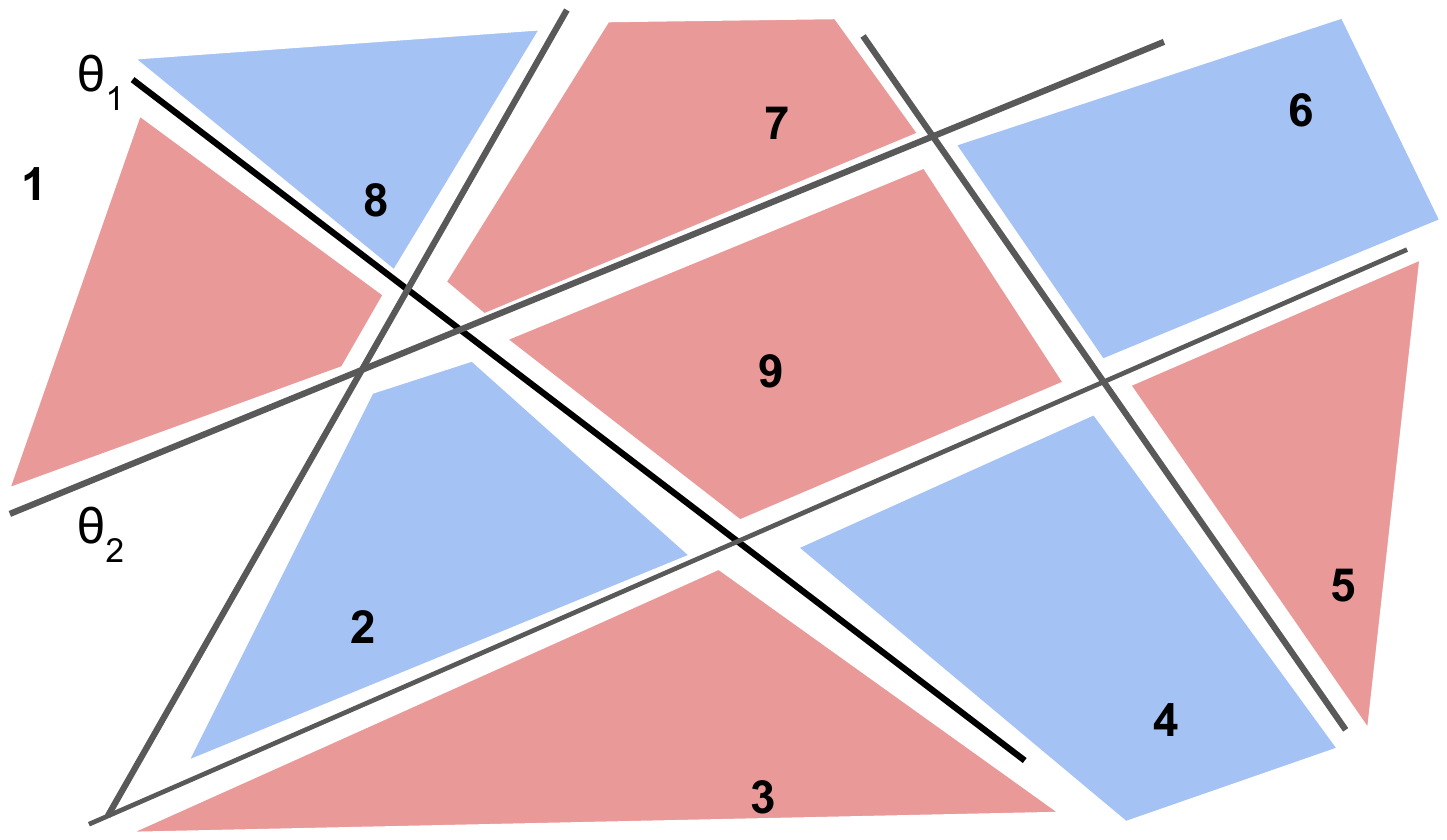} \label{fig:multiclass}}
  \caption{Construction of a log-linear model that breaks log-linear guardedness.
  \setlength{\belowcaptionskip}{-15pt}
  }
  \label{fig:illustration}
\end{figure*}

Our analysis is based on the following simple observation (see \cref{lemma:composition} in the Appendix) that the composition of two $\delta$-discretized log-linear models is itself a $\delta$-discretized log-linear model.
Using this fact, we show that when $|\yset| = |\zset| = 2$,
and the predictive family is the set of $\delta$-discretized binary log-linear models, $\varepsilon$-guarded representations $h(\X)$ cannot leak information through a downstream classifier.\looseness=-1

\begin{theorem}
Let $\Vfamdelta$ be the family of $\delta$-discretized log-linear models, and let $\X$ be a representation-valued random variable.
Define $\YHAT$ as in \Cref{eq:binary-yhat}, then $\Ivdeltahxz < \varepsilon$ implies $\Ivdeltayz < \varepsilon$.\looseness=-2

\label{claim:binary-y-v-info}
\end{theorem}

\begin{proof}
Define the hard thresholding function 
\begin{equation}\label{eq:threshhold}
    \tau(x) \defeq \begin{cases} 1, &\textbf{if}\quad x > 0 \\
    0, &\textbf{else}
    \end{cases}
\end{equation}
We assume, by way of contradiction, that $\Ivdeltahxz < \varepsilon$, but $\Ivdeltayz > \varepsilon$. 
We start by algebraically manipulating $\Ivdeltayz$:
\begin{align}
&\mi_{\mathcal{V}^\delta}( \mathrm{\widehat{Y}} \to \Z) = \CHvdelta{\Z} - \CHvdelta{\Z\mid \YHAT} \nonumber \\
 &=\CHvdelta{\Z} + \sup_{q \in \Vfamdelta} \E{(z, y) \sim p} {\!\!\!\!\log q(z \mid y)} \\
  &=\CHvdelta{\Z} + \sup_{q \in \Vfamdelta} \E{(z, \xx) \sim p} { \!\!\!\!\log q(z \mid \tau(\vtheta^{\top} h(\xx) + \phi))} \nonumber
\end{align}
for some $\vtheta$ and $\phi$ as in the definition of $t$ in \Cref{eq:binary-yhat}.
Now, by \cref{lemma:composition}, we note that, for all $q \in \Vfamdelta$, there exists a classifier $r \in \Vfamdelta$ such that $r(z \mid h(\xx)) = q(z \mid \tau(\vtheta^{\top}h(\xx) + \phi))$. 
This implies that $\Ivdeltahxz \geq \mi_{\mathcal{V}^\delta}( \mathrm{\widehat{Y}} \to \Z)  > \varepsilon$,\footnote{
\Cref{lemma:composition} only guarantees that a classifier of the form $q(z \mid \tau(\vtheta^{\top} h(\xx) + \phi)$, where $q \in \Vdelta$, can be converted into a classifier $r(z\mid h(\xx)) \in \Vdelta$.
However, we have no proof of the opposite implication. 
Hence, we have only shown $\Ivdeltahxz \geq \mi_{\mathcal{V}^\delta}( \mathrm{\widehat{Y}} \to \Z)$.
} contradicting the assumption that $\Ivdeltahxz < \varepsilon$.
Thus, $\Ivdeltayz < \varepsilon$, as desired.\looseness=-1
\end{proof}

\subsection{A Multiclass Downstream Classifier}
\label{sec:multiclass}

The above discussion shows that when both $\Z$ and $\Y$ are binary, $\varepsilon$-log-linear guardedness with respect to the family of discretized log-linear models (\cref{def:discrete-logistic}) implies limited leakage of information about $\Z$ from $\YHAT$. 
It was previously implied \cite{ravfogel-etal-2020-null,elazar2021amnesic} that linear concept erasure prevents information leakage about $\Z$ through the labeling of a log-linear classifier $\YHAT$, i.e., it was assumed that \cref{claim:binary-y-v-info} in \cref{sec:binary} can be generalized to the multiclass case.
Specifically, it was argued that a subsequent linear layer, such as the linear language-modeling head, would not be able to recover the information because it is linear.
In this paper, however, we note a key flaw in this argument.
If the data is log-linearly guarded, then it is easy to see that the \emph{logits}, which are a linear transformation of the guarded representation, cannot encode the information.
However, multiclass classification is usually performed by a softmax classifier, which adds a non-linearity.
Note that the decision boundary of the softmax classifier for every pair of labels is linear since class $i$ will have higher softmax probability than class $j$ if, and only if, $(\ttheta_i-\ttheta_j)^{\top}\xx > 0$.

Next, we demonstrate that this is enough to break guardedness. 
We start with an example. 
Consider the data in $\R^2$ presented in \cref{fig:sep}, where the distribution $p(\X, \Z)$ has 4 distinct clusters, each with a different label from $\zset$, corresponding to Voronoi regions \cite{voronoi1908nouvelles} formed by the intersection of the axes. 
The \redtext{red} clusters correspond to $\Z = \top$ and the \bluetext{blue} clusters correspond to $\Z = \bot$.
The data is taken to be log-linearly guarded with respect to $\Z$.\footnote{Information-theoretic guardedness depends on the density over $p(\X)$, which is \emph{not} depicted in the illustrations in \cref{fig:sep}.}
Importantly, we note that knowledge of the quadrant (i.e., the value of $\Y$), renders $\Z$ recoverable by a 4-class log-linear model.

Assume the parameter matrix $\boldsymbol{\Theta} \in \R^{4 \times 2}$ of this classifier is composed of four columns $\boldsymbol{\theta}_1, \boldsymbol{\theta}_2, \boldsymbol{\theta}_3, \boldsymbol{\theta}_4$ such that $\boldsymbol{\theta}_1 = \alpha \cdot  [1,1]^{\top}, \boldsymbol{\theta}_2 = \alpha \cdot [-1,1]^{\top}, \boldsymbol{\theta}_3=\alpha \cdot [-1,-1]^{\top}, \boldsymbol{\theta}_4 = \alpha \cdot [1,-1]^{\top}$, for some $\alpha > 0$.
These directions encode the quadrant of a point: When the norm of the parameter vectors is large enough, i.e., for a large enough $\alpha$, the probability of class $i$ under a log-linear model will be arbitrarily close to 1 if, and only if, the input is in the $i^\text{th}$ quadrant and arbitrarily close to 0 otherwise.
Given the information about the quadrant, the data is rendered \emph{perfectly linearly separable}.
Thus, the labels $\YHAT$ predicted by a multiclass softmax classifier can recover the linear separation according to $\Z$.\looseness=-1 

This argument can be generalized to a separation that is not axis-aligned (\cref{fig:multiclass}). 

\begin{definition}
Let $\ttheta_1, \dots, \ttheta_K$ be column vectors orthogonal to corresponding linear subspaces, and let $R_1, \dots, R_M$ be the Voronoi regions formed by their intersection (\cref{fig:multiclass}). Let $p(\X,\Z)$ be any data distribution such that any two points in the same region have the same value of $\Z$:\looseness=-1
\begin{equation}
\xx_i \in R_k \wedge \xx_j \in R_k \Longrightarrow z_i = z_j
\end{equation}
for all $(\xx_i,z_i), (\xx_j,z_j) \sim
p(\X,\Z)$ and for all Voronoi regions $R_k$.
We call such distribution a \defn{$K$-Voronoi} distribution.
\label{def:voronoi}
\end{definition}

\begin{theorem}
Fix $\varepsilon > 0$.
Let $p(\X,\Z)$ be a $K$-Voronoi distribution, and let $h$ linearly $\varepsilon$-guard $\X$ against $\Z$ with respect to the family $\Vfam$ of log-linear models.
Then, for every $\eta>0$, there exists a $K$-class log-linear model such that $\Ivyz > 1- \eta$.\footnote{In base two, $\log_2(2) = 1$ is the maximum achievable $\Vfam$-information value in the binary case.}\looseness=-1

\label{lem:softmax}
\end{theorem}

\begin{proof}
 By assumption, the support of $p(\X)$ is divided up into $K$ Voronoi regions, each with a label from $\zset$.
 See \cref{fig:illustration} for an illustrative example.

Define the region identifier $\indk(i)$ for each region $k$ as follows\looseness=-1
\begin{equation}
     \indki \defeq  
\begin{cases}
   1,& \textbf{if } \ttheta_i^{\top} \xx > 0 \textbf{ for } \xx \in R_k\\
     -1,              & \textbf{if } \ttheta_i^{\top} \xx < 0 \textbf{ for } \xx \in R_k
\end{cases}
\end{equation}
We make the simplifying assumption that points $\xx$ that lie on line $\ttheta_i^{\top} \xx$ for any $i$ occur with probability zero.
Consider a $K$-class log-linear model with a parameter matrix $\vThetastar \in \R^{D \times K}$ that contains, in its $j^{\text{th}}$ column, the vector $\vtheta_j^\star \defeq \alpha \sum_{k=0}^K \indjk \vtheta_k$, i.e., we sum over all $\vtheta_k$ and give positive weight to a vector $\vtheta_k$ if a positive dot product with it is a necessary condition for a point $\xx$ to belong to the $k^{\text{th}}$ Voronoi region.
Additionally, we scale the parameter vector by some $\alpha > 0$.
Let $\xx \in R_j$ and let $R_m$ be a Voronoi region such that $j \neq m$. We next inspect the ratio 
\begin{subequations}
\begin{align}
 r(\alpha) &\defeq \frac{\softmax({\vThetastar}^{\top}\xx)_j}{\,\,\softmax({\vThetastar}^{\top}\xx)_m}\\
&= e^{\left({\ttheta^{\star}}_j - \ttheta_m^\star\right)^{\top} \xx} \\
    &= e^{ {(\alpha \sum_{k=0}^K \indkj \ttheta_k - \alpha \sum_{k=0}^K \indkm \ttheta_k)}^{\top} \xx} \nonumber \\
   &= e^{ \alpha {(\sum_{k=0}^K (\indkj-\indkm)\ttheta_k)}^{\top} \xx } 
\end{align}
\end{subequations}
We now show that $\alpha (\sum_{k=0}^K (\indkj-\indkm)\ttheta_k)^{\top} \xx > 0$ through the consideration of the following three cases:
\begin{itemize}
    \item \textbf{Case 1}: $\indkj=\indkm$. In this case, the subspace $\ttheta_k$ is a necessary condition for belonging to both regions $j$ and $m$. Thus, the summand is zero.
    \item \textbf{Case 2}: $\indkj=1$, but $\indkm=-1$. In this case, $\indkj-\indkm=2$. As $\xx \in R_j$, we know that ${\ttheta_k}^{\top} \xx > 0$, and the summand is positive.
    \item \textbf{Case 3}: $\indkj=-1$, but $\indkm=1$. In this case, $\indkj-\indkm=-2$. As $\xx \in R_j$, we know that ${\ttheta_k}^{\top} \xx < 0$, and the summand is, again, positive.
\end{itemize}
Since $j \neq m$, a summand corresponding to cases 2 and 3 must occur. Thus, the sum is strictly positive.
It follows that $\lim_{\alpha \rightarrow \infty} r(\alpha) = 1$.
Finally, for $\YHAT$ defined as in \cref{eq:multi-yhat}, we have $p(\YHAT = j \mid \xx \in R_j) = 1$ for $\alpha$ large.
Now, because all points in each $R_j$ have a distinct label from $\zset$, it is trivial to construct a binary log-linear model that places arbitrarily high probability on $R_j$'s label, which gives us $\Ivyz > 1-\eta$ for all $\eta>0$ small.
This completes the proof.

\end{proof}

This construction demonstrates that one should be cautious when arguing about the implications of log-linear guardedness when multiclass softmax classifiers are applied over the guarded representations. When log-linear guardedness with respect to a binary $\Z$ is imposed, there may still exist a set of $k>2$ linear separators that separate $\Z$.\looseness=-1

\section{Accuracy-Based Guardedness}
We now define an accuracy-based notion of guardedness and discuss its implications.
Note that the information-theoretic notion of guardedness described above does not directly imply that the accuracy of the log-linear model is damaged.
To see this, consider a binary log-linear model on balanced data that always assigns a probability of $\frac{1}{2} + \Delta$ to the correct label and $\frac{1}{2} - \Delta$ to the incorrect label.
For small enough $\Delta$, the cross-entropy loss of such a classifier will be arbitrarily close to the entropy $\log(2)$, even though it has perfect accuracy.
This disparity motivates an accuracy-based notion of guardedness.\looseness=-1

We first define the \defn{accuracy function} $\loss$ as
\begin{align}
\loss(q, \xx, z) = \begin{cases}
1, & \textbf{if  } \underset{z' \in \zset}{\,\argmax}\,q(z' \mid \xx) = z \\
0, & \textbf{else}
\end{cases}
\end{align}
The \defn{conditional $\Vfam$-accuracy} is defined as
\begin{equation}\label{eq:v-conditional-accuracy}
    \CAv\left(\Z \mid \X\right) \defeq \sup_{q \in \Vfam} \E{(z, \xx) \sim p(\Z, \X)} \ell(q, \xx, z)
\end{equation}
The \defn{$\Vfam\text{-accuracy}$} is a special case of \Cref{eq:v-conditional-accuracy} when no random variable is conditioned on\looseness=-1
\begin{equation}
    \CAv\left(\Z\right)\defeq \sup_{q \in \Vfam} \E{z \sim p(\Z)} \ell(q, z)
\end{equation}
where we have overloaded $\ell$ to take only two arguments.
We can now define an analogue of \citeposs{xu2020theory} $\Vfam$-information for accuracy as the difference between the unconditional $\Vfam$-accuracy and the conditional $\Vfam$-accuracy\footnote{Note the order of the terms is reversed in $\Vfam$-accuracy.\looseness=-1}
\begin{equation}
    \Iacc(\X \rightarrow \Z) \defeq \CAv\left(\Z \mid \X\right) - \CAv\left(\Z\right)
    \label{def:v-acc}
\end{equation}
Note that the $\Vfam$-accuracy is bounded below by 0 and above by $\frac{1}{2}$.

\begin{definition}[Accuracy-based $\Vfam$-guardedness]
Let $\X$ be a representation-valued random variable and let $\Z$ be an attribute-valued random variable. 
Moreover, let $\Vfam$ be a predictive family. 
A guarding function $h$ $\varepsilon$-guards $\X$ against $\Z$ with respect to $\Vfam$
if
$\Iacc(h(\X) \rightarrow \Z) < \varepsilon$.\looseness=-1
\end{definition}

\begin{definition}[Accuracy-based Empirical $\Vfam$-guardedness]
Let $\dataset = \{ (\xxn, \zzn) \}_{n=1}^N$ where $(\xxn, \zzn) \sim p(\X, \Z)$.
Let $\Xemp$ and $\Zemp$ be random variables over $\R^D$ and $\zset$, respectively, whose distribution corresponds to the marginals of the empirical distribution over $\dataset$. 
A guarding function $h$ empirically $\varepsilon$-guards $\dataset$ with respect to $\Vfam$
if $\Iacc(h(\Xemp) \rightarrow \Zemp) < \varepsilon$.\looseness=-1
\label{def:acc-guardedness}
\end{definition}

When focusing on accuracy in predicting $\Z$, it is natural to consider the \textbf{independence} (also known as \textbf{demographic parity}) \cite{feldman2015certifying} of the downstream classifiers that are trained over the representations.

\begin{definition}
The $L_1$ \defn{independence gap} measures the difference between the distribution of the model's predictions on the examples for which $\Z=\bot$, and the examples for which $\Z=\top$.
It is formally defined as
\begin{align}
&\text{GAP}_{\text{ind}}(\rv{\widehat{Y}} \to \rv{Z} \mid \X)\label{def:ind} \\
&\quad\defeq \sum_{y \in \yset} \Big| \underset{{\xx \sim p(\X \mid \Z = \bot)}}{\mathbb{E}}  p(\YHAT=y \mid \Z=\bot, \X=\xx)  \nonumber  \\
&\quad\quad\quad - \underset{{\xx \sim p(\X \mid \Z = \top)}}{\mathbb{E}}  p(\YHAT=y \mid \Z=\top, \X=\xx) \Big| \nonumber 
\end{align}
where $p(\X \mid \Z)$ is the conditional distribution over representations given the protected attribute.
\end{definition}

In \Cref{lemma-indp}, we prove that if the data is linearly $\varepsilon$-guarded and \defn{globally balanced} with respect to $\Z$, i.e., if $p(\Z=\bot)=p(\Z= \top) = \frac{1}{2}$, then the prediction of any linear binary downstream classifier is $4\varepsilon$ independent of $\Z$. Note that is true \emph{regardless} of any imbalance of the protected attribute $\Z$ within each class $y \in \yset$ in the downstream task: the data only needs to be globally balanced.
\begin{restatable}{proposition}{independence}
Let $\Vfam$ be the family of binary log-linear models, and assume that $p(\X, \Z)$ is globally balanced, i.e., $p(\Z=\bot)=p(\Z=\top)=\frac{1}{2}$.
Furthermore, let $h$ be a guarding function that $\varepsilon$-guards $\X$ against $\Z$ with respect to $\Vfam$ in terms of accuracy (\cref{def:acc-guardedness}), i.e.,
$\Iacc(h(\X) \rightarrow \Z) < \varepsilon$.
Let $\YHAT$ be defined as in \cref{eq:binary-yhat}.
Then, the $L_1$ independence gap (\cref{def:ind}) satisfies $\gapindh \leq 4 \varepsilon$.
\label{lemma-indp}
\end{restatable}
\begin{proof}
See \cref{app:acc-guardedness-ind} for the proof.
\end{proof}

\section{Experimental Evaluation}
\label{sec:experiments}
In the empirical portion of our paper, we evaluate the extent to which our theory holds in practice.\looseness=-1

\paragraph{Data.} 
We perform experiments on gender bias mitigation on the Bias in Bios dataset \citep{DBLP:journals/corr/abs-1901-09451}, which is composed of short biographies annotated by both gender and profession.
We represent each biography with the \CLS\ representation in the final hidden representation of pre-trained BERT, which creates our representation random variable $\X$.
We then try to guard against the protected attribute gender, which constitutes $\Z$.\looseness=-1

\paragraph{Approximating log-linear guardedness.} 
To approximate the condition of log-linear guardedness, we use RLACE \cite{ravfogel2022linear}. 
The method is based on a minimax game between a log-linear predictor that aims to predict the concept of interest from the representation and an orthogonal projection matrix that aims to guard the representation against prediction.
The process results in an orthogonal projection matrix $P \in \R^{D \times D}$, which, empirically, prevents log-linear models from predicting gender after the linear transformation $P$ is applied to the representations.
This process constitutes our guarding function $\hrlace$.
Our theoretical result (\cref{claim:binary-y-v-info}) only holds for $\delta$-discretized log-linear models. 
RLACE, however, guards against conventional log-linear models.
Thus, we apply $\delta$-discretization post hoc, i.e., after training.

\subsection{Quantifying Empirical Guardedness}

We test whether our theoretical analysis of leakage through binary and multiclass downstream classifiers holds in practice, on held-out data. 
Profession prediction serves as our downstream prediction task (i.e., our $\YHAT$), and we study binary and multiclass variants of this task. In both cases, we measure three $\Vfam$-information estimates: 
\looseness=-1
\begin{itemize}
    \item \textbf{Evaluating $\Ivxz$}. 
    To compute an empirical upper bound on information about the protected attribute which is linearly extractable from the representations, we train a log-linear model to predict $\zn$ from $\xxn$, i.e., from the unguarded representations. In other words, this is an upper bound on the information that could be leaked through the downstream classifier $\YHAT$.
    
    \item \textbf{Evaluating $\Ivyzprof$}.
    We quantify leakage through a downstream classifier $\YHAT$ by estimating $\Ivyz$, for binary and multiclass $\YHAT$, via two different approaches.
    The first of these, denoted $\Ivyzprof$, is computed by training two log-linear and stiching them together into a pipeline. 
    First, we fit a log-linear model on top of the guarded representations $\hrlace(\X)$ to yield predictions for a downstream task $\Yhatprof = t(\hrlace(\X))$ where $t : \R^D \rightarrow \yset$ is the function induced by the trained classifier. 
    Subsequently, we train a second log-linear model $\ZHAT = r(\Yhatprof)$ with $r: \yset \rightarrow \{0, 1\}$ to predict $\Z$ from the output of $\YHAT$
    In words, $\Yhatprof$ represents the argmax from the distribution of profession labels (binary or multiclass).
    We approximate the $\Vfam$-information $\Yhatprof$ leaks about $\Z$ through the cross-entropy loss of a second classifier trained to predict the protected attribute from $\Yhatprof$, i.e., we compute empirical guardedness (\cref{def:information-theoretic-guardedness}) on held-out data.\looseness=-1

   \item \textbf{Evaluating $\Ivyzadv$}.
    In addition to the standard scenario estimated by $\Ivyzprof$, we also ask: What is the maximum amount of information that a downstream classifier could leak about gender?
    $\Ivyzadv$ estimates this quantity, with a variant of the setup of $\Ivyzprof$. 
    Namely, instead of training the two log-linear models separately, we train them together to find the $\Yhatadv$ that is adversarially chosen to predict gender well.
   However, the argmax operation is not differentiable, so we remove it during training.

   In practice, this means $\Yhatadv$ does not predict profession, but instead predicts a latent distribution which is adversarially chosen so as to best enable the prediction of gender.\footnote{Only at inference time do we apply the argmax over the first log-linear model to get a prediction $\Yhatadv = y$. We find that the loss of the composition model is not increased by the argmax operation.}
\end{itemize}

While high $\Ivyzadv$ indicates that there exists an adversarial log-linear model that leaks information about $\Z$, it does not necessarily mean that common classifiers like those used to compute $\Ivyzprof$ would leak that information. 
 Across all of 3 conditions, we explore how different values of the thresholds $\delta$ (applied after training) affect the $\Vfam$-information.
Refer to \cref{app:experiments} for comprehensive details regarding our experimental setup.\looseness=-1

\subsection{Binary $\Z$ and $\Y$}
\label{sec:binary-experiment}

We start by evaluating the case where both $\Z$ and $\Y$ take only two values each.

\paragraph{Experimental Setting.}
To create our set $\yset$ for the binary classification task, we randomly sample 15 pairs of professions from the Bias in Bios dataset; see \cref{app:experiments}.
We train a binary log-linear model to predict each profession from the representation after the application of the RLACE guarding function, $\hrlace(\xx_n)$.
Empirically, we observe that our log-linear models achieve no more than $2\%$ above majority accuracy on the protected data.
For each pair of professions, we estimate three forms of $\Vfam$-information. 

\begin{figure}
    \centering
    \includegraphics[width=1.0\columnwidth]{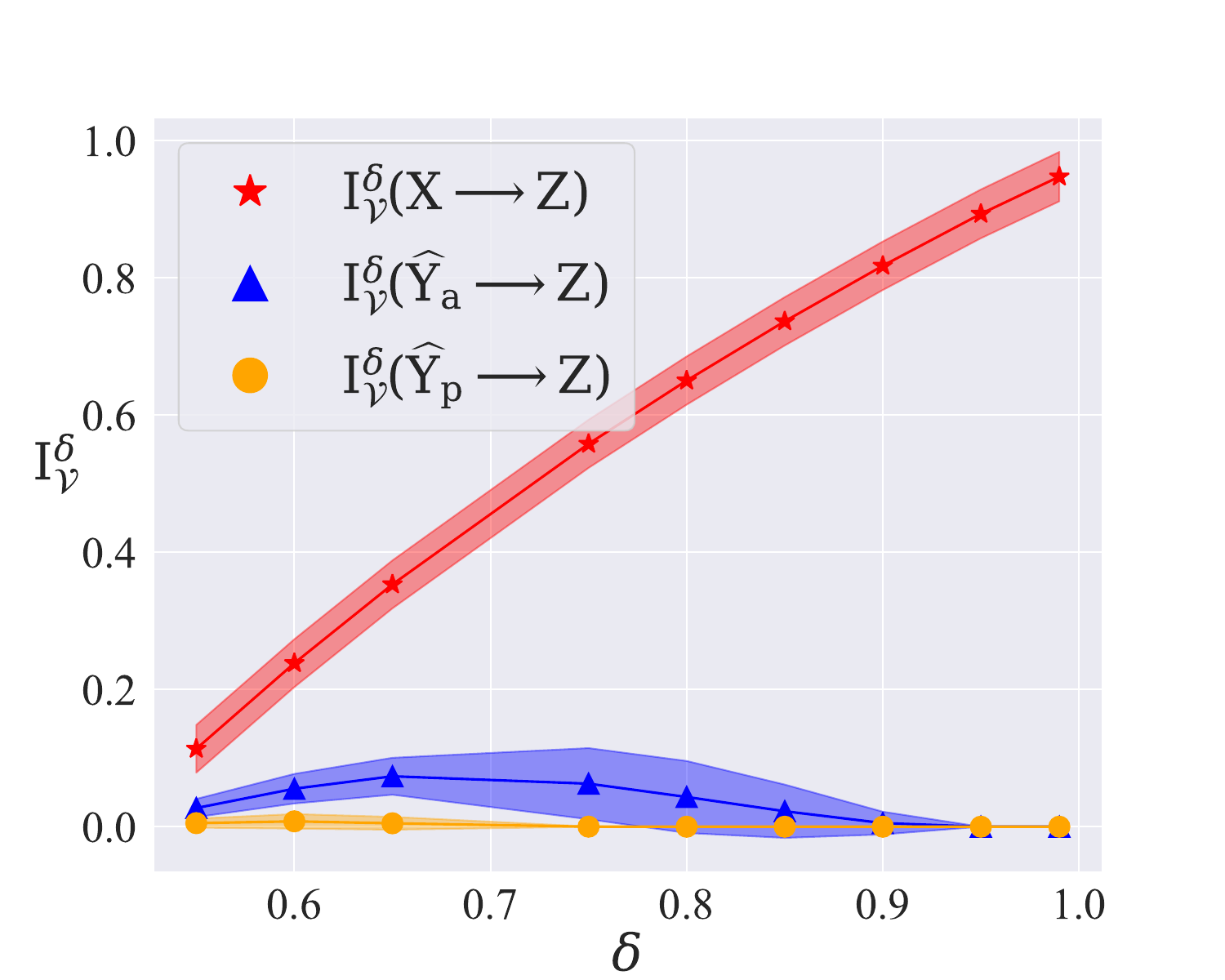}
    \caption{Results for \cref{sec:binary-experiment}. 
    Estimate of $\Vfam$-information between the protected attribute and (1) the original representations (\redtext{red}); (2) the labels induced by the inner model within a composition of two log-linear models, trained to adversarially recover gender (\bluetext{blue}); (3) labels for the downstream task (the predictions of profession classifiers; \orangetext{orange}). The curve is the mean over different pairs of professions, and the shaded area representations 1 standard deviation.
    The $x$-axis presents results for different values of the threshold $\delta$.
    Recall the threshholding is applied post hoc.}
    \label{fig:results-binary}
    \vspace{-.5cm}
\end{figure}

\paragraph{Results.} The results are presented in \cref{fig:results-binary}, for the 15 pairs of professions we experimented with (each curve is the mean over all pairs), the three quantities listed above, and different values of the threshold $\delta$ on the $x$-axis.
Unsurprisingly, we observe that the $\Vfam$-information estimated from the original representations (the \redtext{red} curve) has high values for some thresholds, indicating that BERT representations do encode gender distinctions. The \bluetext{blue} curve, corresponding to $\Ivyzadv$, measures the ability of the adversarially constructed binary downstream classifier to recover the gender information. 
It is lower than the \redtext{red} curve but is nonzero, indicating that the solution found by RLACE does not generalize perfectly.
Finally, the \orangetext{orange} curve, corresponding to $\Ivyzprof$, measures the amount of leakage we get in practice from downstream classifiers that are trained on profession prediction. In that case, the numbers are significantly lower, showing that RLACE does provide decent guardedness in practice.\looseness=-1

\subsection{Binary $\Z$, Multiclass $\Y$}
\label{sec:multi-experiment}
Empirically, we have shown that RLACE provides good, albeit imperfect, protection against binary log-linear model adversaries.
This finding is in line with the conclusions of \cref{claim:binary-y-v-info}. 
We now turn to experiments on multiclass classification, i.e., where $|\yset| > 2$.
According to \cref{sec:multiclass}, to the extent that the $K$-Voronoi assumption holds, we expect guardedness to be broken with a large enough $|\yset|$.\looseness=-1

\paragraph{Experimental Setting.}

Note that, since $|\yset| > 2$, $\YHAT$ is a multiclass log-linear classifier over $\yset$, but the logistic classifier that predicts gender from the argmax over these remains binary. 
We consider different values of $|\yset| =  2,4,8,16,32,64$.

\begin{figure}
    \centering
    \includegraphics[width=1.0\columnwidth]{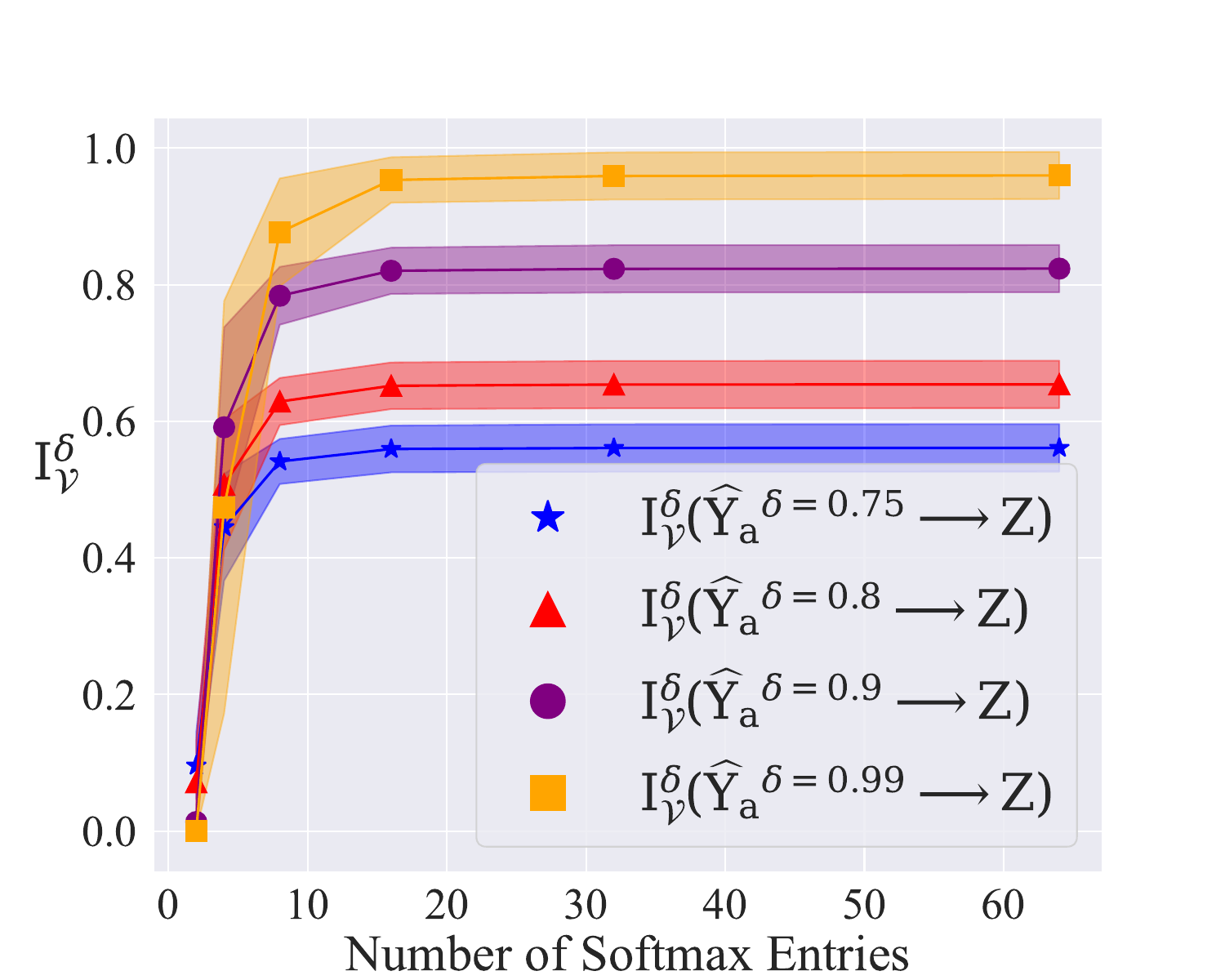}
    \caption{Results for \cref{sec:multi-experiment}.
     Estimate of $\Vfam$-information between the protected attribute and $\Yhatadv$ with various $\delta$.}
    \label{fig:results-multi}
    \vspace{-.8cm}
\end{figure}

\paragraph{Results.}
The results are shown in \cref{fig:results-multi}. For all professions, we find a log-linear model whose predicted labels are highly predictive of the protected attribute. 
Indeed, softmax classifiers with 4 to 8 entries (corresponding to hidden neurons in the network which is the composition of two log-linear models) perfectly recover the gender information.
This indicates that there are labeling schemes of the data using 4 or 8 labels that recover almost all information about $\Z$.

\subsection{Discussion}

Even if a set of representations is log-linearly guarded, one can still adversarially construct a multiclass softmax classifier that recovers the information. These results stem from the disparity between the manifold in which the concept resides, and the expressivity of the (linear) intervention we perform: softmax classifiers can access information that is inaccessible to a purely linear classifier. Thus, interventions that are aimed at achieving guardedness should consider the specific adversary against which one aims to protect. 

\section{Related Work}

Techniques for information removal are generally divided into adversarial methods and post-hoc linear methods. Adversarial methods \citep{edwards2015censoring,xie2017controllable, chen2018adversarial,elazar-goldberg-2018-adversarial, zhang2018mitigating} use a gradient-reversal layer during training to induce representations that do not encode the protected attribute. However, \citet{elazar-goldberg-2018-adversarial} have shown that these methods fail to exhaustively remove all the information associated with the protected attribute. Linear methods have been proposed as a tractable alternative, where one identifies a linear subspace that captures the concept of interest, and neutralizes it using algebraic techniques. Different methods have been proposed for the identification of the subspace, e.g., PCA and variants thereof \citep{bolukbasi2016man, kleindessner2023efficient}, orthogonal rotation \citep{dev2020oscar}, classification-based \citep{ravfogel-etal-2020-null}, spectral \citep{shao2023erasure, shao-etal-2023-gold} and adversarial approaches \citep{ravfogel2022linear}.

Different definitions have been proposed for fairness \citep{mehrabi2021survey}, but they are mostly extrinsic---they concern themselves only with the predictions of the model, and not with its representation space. Intrinsic bias measures, which focus on the representation space of the model, have been also extensively studied. These measures quantify, for instance, the extent to which the word representation space encodes gender distinctions \citep{bolukbasi2016man,caliskan2017semantics,kurita2019measuring,zhao2019gender}. The \emph{relation} between extrinsic and intrinsic bias measures is understudied, but recent works have demonstrated empirically either a relatively weak or inconsistent correlation between the two \citep{goldfarb2021intrinsic, orgad2022gender,cao2022intrinsic, orgadchoose, steed-etal-2022-upstream, shen-etal-2022-representational, cabello2023independence}.\looseness=-1

\section{Conclusion}

We have formulated the notion of guardedness as the inability to \emph{directly} predict a concept from the representation. We show that log-linear guardedness with respect to a binary protected attribute does not prevent a \emph{subsequent} multiclass linear classifier trained over the guarded representations from leaking information on the protected attribute. In contrast, when the main task is binary, we can bound that leakage. Altogether, our analysis suggests that the deployment of linear erasure methods should carefully take into account the manner in which the modified representations are being used later on, e.g., in classification tasks.  

\section*{Limitations}

Our theoretical analysis targets a specific notion of information leakage, and it is likely that it does not apply to alternative ones. While the $\Vfam$-information-based approach seems natural, future work should consider alternative extrinsic bias measures as well as alternative notions of guardedness. Additionally, our focus is on the linear case, which is tractable and important---but limits the generality of our conclusions. We hope to extend this analysis to other predictive families in future work.

\section*{Ethical Considerations}
The empirical experiments in this work involve the removal of binary gender information from a pre-trained representation. Beyond the fact that gender is a non-binary concept, this task may have real-world applications, in particular such that relate to fairness. We would thus like to remind the readers to take the results with a grain of salt and be extra careful when attempting to deploy methods such as the one discussed here. 
Regardless of any theoretical result, care should be taken to measure the effectiveness of bias mitigation efforts in the context in which they are to be deployed, considering, among other things, the exact data to be used and the exact fairness metrics under consideration.

\section*{Acknowledgements}
We thank Afra Amini, Clément Guerner, David Schneider-Joseph, Nora Belrose and Stella Biderman for their thoughtful comments and revision of this paper. This project received funding from the Europoean Research Council (ERC) under the Europoean Union's Horizon 2020 research and innovation program, grant agreement No. 802774 (iEXTRACT). Shauli Ravfogel is grateful to be supported by the Bloomberg Data Science Ph.D Fellowship.
Ryan Cotterell acknowledges the Google Research Scholar program for supporting the proposal  ``Controlling and Understanding Representations through Concept Erasure.''

\bibliography{anthology,custom}

\begin{thebibliography}{37}
\expandafter\ifx\csname natexlab\endcsname\relax\def\natexlab#1{#1}\fi

\bibitem[{Belinkov et~al.(2017)Belinkov, M{\`a}rquez, Sajjad, Durrani, Dalvi, and Glass}]{belinkov2017evaluating}
Yonatan Belinkov, Llu{\'\i}s M{\`a}rquez, Hassan Sajjad, Nadir Durrani, Fahim Dalvi, and James Glass. 2017.
\newblock \href {https://aclanthology.org/I17-1001} {Evaluating layers of representation in neural machine translation on part-of-speech and semantic tagging tasks}.
\newblock In \emph{Proceedings of the Eighth International Joint Conference on Natural Language Processing (Volume 1: Long Papers)}, pages 1--10, Taipei, Taiwan. Asian Federation of Natural Language Processing.

\bibitem[{Belrose et~al.(2023)Belrose, Schneider-Joseph, Ravfogel, Cotterell, Raff, and Biderman}]{belrose2023leace}
Nora Belrose, David Schneider-Joseph, Shauli Ravfogel, Ryan Cotterell, Edward Raff, and Stella Biderman. 2023.
\newblock \href {https://arxiv.org/abs/2306.03819} {{LEACE}: Perfect linear concept erasure in closed form}.
\newblock \emph{arXiv preprint arXiv:2306.03819}.

\bibitem[{Bolukbasi et~al.(2016)Bolukbasi, Chang, Zou, Saligrama, and Kalai}]{bolukbasi2016man}
Tolga Bolukbasi, Kai-Wei Chang, James~Y. Zou, Venkatesh Saligrama, and Adam~T. Kalai. 2016.
\newblock \href {https://proceedings.neurips.cc/paper/2016/file/a486cd07e4ac3d270571622f4f316ec5-Paper.pdf} {Man is to computer programmer as woman is to homemaker? {D}ebiasing word embeddings}.
\newblock \emph{Advances in {N}eural {I}nformation {P}rocessing {S}ystems}, 29:4349--4357.

\bibitem[{Cabello et~al.(2023)Cabello, J{\o}rgensen, and S{\o}gaard}]{cabello2023independence}
Laura Cabello, Anna~Katrine J{\o}rgensen, and Anders S{\o}gaard. 2023.
\newblock \href {https://arxiv.org/abs/2304.10153} {On the independence of association bias and empirical fairness in language models}.
\newblock \emph{arXiv preprint arXiv:2304.10153}.

\bibitem[{Caliskan et~al.(2017)Caliskan, Bryson, and Narayanan}]{caliskan2017semantics}
Aylin Caliskan, Joanna~J. Bryson, and Arvind Narayanan. 2017.
\newblock \href {https://www.science.org/doi/abs/10.1126/science.aal4230?casa_token=DDUBhpRGf88AAAAA:dFDXkLpqDU9X7D1QBAvUxY2gT4f7AnMyHvMur7Rjl52EMQOTabU3XKoUFXnBF3264Egmrd5Dn8t83Eg} {Semantics derived automatically from language corpora contain human-like biases}.
\newblock \emph{Science}, 356(6334):183--186.

\bibitem[{Cao et~al.(2022)Cao, Pruksachatkun, Chang, Gupta, Kumar, Dhamala, and Galstyan}]{cao2022intrinsic}
Yang Cao, Yada Pruksachatkun, Kai-Wei Chang, Rahul Gupta, Varun Kumar, Jwala Dhamala, and Aram Galstyan. 2022.
\newblock \href {https://aclanthology.org/2022.acl-short.62/} {On the intrinsic and extrinsic fairness evaluation metrics for contextualized language representations}.
\newblock In \emph{Proceedings of the 60th Annual Meeting of the Association for Computational Linguistics (Volume 2: Short Papers)}, pages 561--570.

\bibitem[{Chen et~al.(2018)Chen, Sun, Athiwaratkun, Cardie, and Weinberger}]{chen2018adversarial}
Xilun Chen, Yu~Sun, Ben Athiwaratkun, Claire Cardie, and Kilian Weinberger. 2018.
\newblock \href {https://transacl.org/index.php/tacl/article/view/1413} {Adversarial deep averaging networks for cross-lingual sentiment classification}.
\newblock \emph{Transactions of the Association for Computational Linguistics}, 6:557--570.

\bibitem[{Cover and Thomas(2006)}]{cover1991elements}
Thomas~M. Cover and Joy~A. Thomas. 2006.
\newblock \href {https://doi.org/10.1002/047174882X} {\emph{Elements of Information Theory}}, 2 edition.
\newblock Wiley-Interscience.

\bibitem[{De{-}Arteaga et~al.(2019)De{-}Arteaga, Romanov, Wallach, Chayes, Borgs, Chouldechova, Geyik, Kenthapadi, and Kalai}]{DBLP:journals/corr/abs-1901-09451}
Maria De{-}Arteaga, Alexey Romanov, Hanna Wallach, Jennifer Chayes, Christian Borgs, Alexandra Chouldechova, Sahin Geyik, Krishnaram Kenthapadi, and Adam~Tauman Kalai. 2019.
\newblock \href {https://dl.acm.org/doi/abs/10.1145/3287560.3287572?casa_token=flqHAHdujPYAAAAA:vrjfe1MlX0NAPc-9kFHkzkoLT35ATUN3dHa-QwJ3roxD4Ot68bNbee4TkTDHQc6-iIkCDtYWrGs5ww} {Bias in bios: A case study of semantic representation bias in a high-stakes setting}.
\newblock In \emph{proceedings of the Conference on Fairness, Accountability, and Transparency}, pages 120--128.

\bibitem[{Dev et~al.(2021)Dev, Li, Phillips, and Srikumar}]{dev2020oscar}
Sunipa Dev, Tao Li, Jeff~M Phillips, and Vivek Srikumar. 2021.
\newblock \href {https://aclanthology.org/2021.emnlp-main.411/} {{OSCaR}: Orthogonal subspace correction and rectification of biases in word embeddings}.
\newblock In \emph{Proceedings of the 2021 Conference on Empirical Methods in Natural Language Processing}, pages 5034--5050.

\bibitem[{Dev and Phillips(2019)}]{dev2019attenuating}
Sunipa Dev and Jeff Phillips. 2019.
\newblock \href {http://proceedings.mlr.press/v89/dev19a.html} {Attenuating bias in word vectors}.
\newblock In \emph{The 22nd International Conference on Artificial Intelligence and Statistics}, pages 879--887. PMLR.

\bibitem[{Edwards and Storkey(2016)}]{edwards2015censoring}
Harrison Edwards and Amos~J. Storkey. 2016.
\newblock \href {http://arxiv.org/abs/1511.05897} {Censoring representations with an adversary}.
\newblock In \emph{4$^\text{th}$ International Conference on Learning Representations}.

\bibitem[{Elazar and Goldberg(2018)}]{elazar-goldberg-2018-adversarial}
Yanai Elazar and Yoav Goldberg. 2018.
\newblock \href {https://doi.org/10.18653/v1/D18-1002} {Adversarial removal of demographic attributes from text data}.
\newblock In \emph{Proceedings of the 2018 Conference on Empirical Methods in Natural Language Processing}, pages 11--21, Brussels, Belgium. Association for Computational Linguistics.

\bibitem[{Elazar et~al.(2021)Elazar, Ravfogel, Jacovi, and Goldberg}]{elazar2021amnesic}
Yanai Elazar, Shauli Ravfogel, Alon Jacovi, and Yoav Goldberg. 2021.
\newblock \href {https://doi.org/10.1162/tacl_a_00359} {Amnesic probing: Behavioral explanation with amnesic counterfactuals}.
\newblock \emph{Transactions of the Association for Computational Linguistics}, 9:160--175.

\bibitem[{Feldman et~al.(2015)Feldman, Friedler, Moeller, Scheidegger, and Venkatasubramanian}]{feldman2015certifying}
Michael Feldman, Sorelle~A Friedler, John Moeller, Carlos Scheidegger, and Suresh Venkatasubramanian. 2015.
\newblock \href {https://dl.acm.org/doi/10.1145/2783258.2783311} {Certifying and removing disparate impact}.
\newblock In \emph{Proceedings of the 21th ACM SIGKDD International Conference on Knowledge Discovery and Data Mining}, pages 259--268.

\bibitem[{Goldfarb-Tarrant et~al.(2021)Goldfarb-Tarrant, Marchant, S{\'a}nchez, Pandya, and Lopez}]{goldfarb2021intrinsic}
Seraphina Goldfarb-Tarrant, Rebecca Marchant, Ricardo~Mu{\~n}oz S{\'a}nchez, Mugdha Pandya, and Adam Lopez. 2021.
\newblock \href {https://aclanthology.org/2021.acl-long.150/} {Intrinsic bias metrics do not correlate with application bias}.
\newblock In \emph{Proceedings of the 59th Annual Meeting of the Association for Computational Linguistics and the 11th International Joint Conference on Natural Language Processing (Volume 1: Long Papers)}, pages 1926--1940.

\bibitem[{Jacovi et~al.(2021)Jacovi, Swayamdipta, Ravfogel, Elazar, Choi, and Goldberg}]{jacovi2021contrastive}
Alon Jacovi, Swabha Swayamdipta, Shauli Ravfogel, Yanai Elazar, Yejin Choi, and Yoav Goldberg. 2021.
\newblock \href {https://doi.org/10.18653/v1/2021.emnlp-main.120} {Contrastive explanations for model interpretability}.
\newblock In \emph{Proceedings of the 2021 Conference on Empirical Methods in Natural Language Processing}, pages 1597--1611, Online and Punta Cana, Dominican Republic. Association for Computational Linguistics.

\bibitem[{Kaneko and Bollegala(2021)}]{kaneko2021debiasing}
Masahiro Kaneko and Danushka Bollegala. 2021.
\newblock \href {https://aclanthology.org/2021.eacl-main.107/} {Debiasing pre-trained contextualised embeddings}.
\newblock In \emph{Proceedings of the 16th Conference of the European Chapter of the Association for Computational Linguistics: Main Volume}, pages 1256--1266.

\bibitem[{Kleindessner et~al.(2023)Kleindessner, Donini, Russell, and Zafar}]{kleindessner2023efficient}
Matth{\"a}us Kleindessner, Michele Donini, Chris Russell, and Muhammad~Bilal Zafar. 2023.
\newblock \href {https://arxiv.org/abs/2302.13319} {Efficient fair {PCA} for fair representation learning}.
\newblock In \emph{International Conference on Artificial Intelligence and Statistics}, pages 5250--5270.

\bibitem[{Kurita et~al.(2019)Kurita, Vyas, Pareek, Black, and Tsvetkov}]{kurita2019measuring}
Keita Kurita, Nidhi Vyas, Ayush Pareek, Alan~W Black, and Yulia Tsvetkov. 2019.
\newblock \href {https://aclanthology.org/W19-3823/} {Measuring bias in contextualized word representations}.
\newblock In \emph{Proceedings of the First Workshop on Gender Bias in Natural Language Processing}, pages 166--172.

\bibitem[{Linzen et~al.(2016)Linzen, Dupoux, and Goldberg}]{linzen2016assessing}
Tal Linzen, Emmanuel Dupoux, and Yoav Goldberg. 2016.
\newblock \href {https://direct.mit.edu/tacl/article/doi/10.1162/tacl_a_00115/43378/Assessing-the-Ability-of-LSTMs-to-Learn-Syntax} {Assessing the ability of {LSTMs} to learn syntax-sensitive dependencies}.
\newblock \emph{Transactions of the Association for Computational Linguistics}, 4:521--535.

\bibitem[{Mehrabi et~al.(2021)Mehrabi, Morstatter, Saxena, Lerman, and Galstyan}]{mehrabi2021survey}
Ninareh Mehrabi, Fred Morstatter, Nripsuta Saxena, Kristina Lerman, and Aram Galstyan. 2021.
\newblock \href {https://dl.acm.org/doi/abs/10.1145/3457607} {A survey on bias and fairness in machine learning}.
\newblock \emph{ACM Computing Surveys}, 54(6):1--35.

\bibitem[{Orgad and Belinkov(2022)}]{orgadchoose}
Hadas Orgad and Yonatan Belinkov. 2022.
\newblock \href {https://www.cs.technion.ac.il/~belinkov/assets/pdf/genbnlp2022.pdf} {Choose your lenses: Flaws in gender bias evaluation}.
\newblock \emph{arXiv preprint arXiv:2210.11471}.

\bibitem[{Orgad et~al.(2022)Orgad, Goldfarb-Tarrant, and Belinkov}]{orgad2022gender}
Hadas Orgad, Seraphina Goldfarb-Tarrant, and Yonatan Belinkov. 2022.
\newblock \href {https://arxiv.org/abs/2204.06827} {How gender debiasing affects internal model representations, and why it matters}.
\newblock \emph{arXiv preprint arXiv:2204.06827}.

\bibitem[{Ravfogel et~al.(2020)Ravfogel, Elazar, Gonen, Twiton, and Goldberg}]{ravfogel-etal-2020-null}
Shauli Ravfogel, Yanai Elazar, Hila Gonen, Michael Twiton, and Yoav Goldberg. 2020.
\newblock \href {https://doi.org/10.18653/v1/2020.acl-main.647} {Null it out: Guarding protected attributes by iterative nullspace projection}.
\newblock In \emph{Proceedings of the 58th Annual Meeting of the Association for Computational Linguistics}, pages 7237--7256, Online. Association for Computational Linguistics.

\bibitem[{Ravfogel et~al.(2022{\natexlab{a}})Ravfogel, Twiton, Goldberg, and Cotterell}]{ravfogel2022linear}
Shauli Ravfogel, Michael Twiton, Yoav Goldberg, and Ryan Cotterell. 2022{\natexlab{a}}.
\newblock \href {https://proceedings.mlr.press/v162/ravfogel22a.html} {Linear adversarial concept erasure}.
\newblock In \emph{International Conference on Machine Learning}, pages 18400--18421. PMLR.

\bibitem[{Ravfogel et~al.(2022{\natexlab{b}})Ravfogel, Vargas, Goldberg, and Cotterell}]{ravfogel2022adversarial}
Shauli Ravfogel, Francisco Vargas, Yoav Goldberg, and Ryan Cotterell. 2022{\natexlab{b}}.
\newblock \href {https://aclanthology.org/2022.emnlp-main.405} {Kernelized concept erasure}.
\newblock In \emph{Proceedings of the 2022 Conference on Empirical Methods in Natural Language Processing}, pages 6034--6055, Abu Dhabi, United Arab Emirates. Association for Computational Linguistics.

\bibitem[{Shao et~al.(2023{\natexlab{a}})Shao, Ziser, and Cohen}]{shao2023erasure}
Shun Shao, Yftah Ziser, and Shay~B. Cohen. 2023{\natexlab{a}}.
\newblock \href {https://doi.org/10.1162/tacl_a_00558} {Erasure of unaligned attributes from neural representations}.
\newblock \emph{Transactions of the Association for Computational Linguistics}, 11:488--510.

\bibitem[{Shao et~al.(2023{\natexlab{b}})Shao, Ziser, and Cohen}]{shao-etal-2023-gold}
Shun Shao, Yftah Ziser, and Shay~B. Cohen. 2023{\natexlab{b}}.
\newblock \href {https://aclanthology.org/2023.eacl-main.118} {Gold doesn{'}t always glitter: Spectral removal of linear and nonlinear guarded attribute information}.
\newblock In \emph{Proceedings of the 17th Conference of the European Chapter of the Association for Computational Linguistics}, pages 1611--1622, Dubrovnik, Croatia. Association for Computational Linguistics.

\bibitem[{Shen et~al.(2022)Shen, Han, Cohn, Baldwin, and Frermann}]{shen-etal-2022-representational}
Aili Shen, Xudong Han, Trevor Cohn, Timothy Baldwin, and Lea Frermann. 2022.
\newblock \href {https://aclanthology.org/2022.findings-aacl.8} {Does representational fairness imply empirical fairness?}
\newblock In \emph{Findings of the Association for Computational Linguistics: AACL-IJCNLP 2022}, pages 81--95, Online only. Association for Computational Linguistics.

\bibitem[{Steed et~al.(2022)Steed, Panda, Kobren, and Wick}]{steed-etal-2022-upstream}
Ryan Steed, Swetasudha Panda, Ari Kobren, and Michael Wick. 2022.
\newblock \href {https://doi.org/10.18653/v1/2022.acl-long.247} {Upstream mitigation is \emph{Not} all you need: {T}esting the bias transfer hypothesis in pre-trained language models}.
\newblock In \emph{Proceedings of the 60th Annual Meeting of the Association for Computational Linguistics (Volume 1: Long Papers)}, pages 3524--3542, Dublin, Ireland. Association for Computational Linguistics.

\bibitem[{Voronoi(1908)}]{voronoi1908nouvelles}
Georges Voronoi. 1908.
\newblock \href {https://www.degruyter.com/document/doi/10.1515/crll.1908.134.198/html?lang=en} {Nouvelles applications des param{\`e}tres continus {\`a} la th{\'e}orie des formes quadratiques. deuxi{\`e}me m{\'e}moire. recherches sur les parall{\'e}llo{\`e}dres primitifs.}
\newblock \emph{Journal f{\"u}r die reine und angewandte Mathematik}, 1908(134):198--287.

\bibitem[{Vylomova et~al.(2017)Vylomova, Cohn, He, and Haffari}]{vylomova2017word}
Ekaterina Vylomova, Trevor Cohn, Xuanli He, and Gholamreza Haffari. 2017.
\newblock \href {https://doi.org/10.18653/v1/W17-4115} {Word representation models for morphologically rich languages in neural machine translation}.
\newblock In \emph{Proceedings of the First Workshop on Subword and Character Level Models in {NLP}}, pages 103--108, Copenhagen, Denmark. Association for Computational Linguistics.

\bibitem[{Xie et~al.(2017)Xie, Dai, Du, Hovy, and Neubig}]{xie2017controllable}
Qizhe Xie, Zihang Dai, Yulun Du, Eduard Hovy, and Graham Neubig. 2017.
\newblock \href {https://proceedings.neurips.cc/paper/2017/file/8cb22bdd0b7ba1ab13d742e22eed8da2-Paper.pdf} {Controllable invariance through adversarial feature learning}.
\newblock In \emph{Advances in Neural Information Processing Systems}, volume~30. Curran Associates, Inc.

\bibitem[{Xu et~al.(2020)Xu, Zhao, Song, Stewart, and Ermon}]{xu2020theory}
Yilun Xu, Shengjia Zhao, Jiaming Song, Russell Stewart, and Stefano Ermon. 2020.
\newblock \href {https://openreview.net/forum?id=r1eBeyHFDH} {A theory of usable information under computational constraints}.
\newblock In \emph{International Conference on Learning Representations}.

\bibitem[{Zhang et~al.(2018)Zhang, Lemoine, and Mitchell}]{zhang2018mitigating}
Brian~Hu Zhang, Blake Lemoine, and Margaret Mitchell. 2018.
\newblock \href {https://dl.acm.org/doi/abs/10.1145/3278721.3278779} {Mitigating unwanted biases with adversarial learning}.
\newblock In \emph{Proceedings of the 2018 AAAI/ACM Conference on AI, Ethics, and Society}, pages 335--340.

\bibitem[{Zhao et~al.(2019)Zhao, Wang, Yatskar, Cotterell, Ordonez, and Chang}]{zhao2019gender}
Jieyu Zhao, Tianlu Wang, Mark Yatskar, Ryan Cotterell, Vicente Ordonez, and Kai-Wei Chang. 2019.
\newblock \href {https://doi.org/10.18653/v1/N19-1064} {Gender bias in contextualized word embeddings}.
\newblock In \emph{Proceedings of the 2019 Conference of the {N}orth {A}merican Chapter of the Association for Computational Linguistics: Human Language Technologies, Volume 1 (Long and Short Papers)}, pages 629--634, Minneapolis, Minnesota. Association for Computational Linguistics.

\end{thebibliography}
\bibliographystyle{acl_natbib}

\newpage
\clearpage

\appendix
\onecolumn


\section{Appendix}
\label{appendix}

\subsection{Composition of $\delta$-discretized binary log-linear models}
\label{app:composition}
\begin{restatable}{lemma}{composition}
Let $\Vfamdelta$ be the family of discretized binary log-linear models (\cref{def:discrete-logistic}).
Let $\tau(\vtheta^{\top} \xx + \phi)$ be a linear decision rule where $\tau$ is defined as in \Cref{eq:threshhold}, and, furthermore, assume $\vtheta^{\top} \xx + \phi \neq 0$ for all $\xx$.
Then, for any $\alpha, \beta \in \R$, there exists a function $r \in \Vfamdelta$ such that $r(0) = \rhodelta(\sigma(\alpha \cdot \tau(\vtheta^{\top} \xx + \phi) + \beta))$ where $\rhodelta$ is defined as in \Cref{def:discrete-logistic}.\looseness=-1
\label{lemma:composition}
\end{restatable}
\begin{proof}
Consider the function composition $ \sigma(\alpha \cdot \tau(\vtheta^{\top} \xx + \phi) + \beta)$.
First, note that  $\tau(\vtheta^{\top} \xx + \phi)$ is a step-function.
And, thus, so, too, is $\sigma(\alpha \cdot \tau(\vtheta^{\top} \xx + \phi)+ \beta)$, i.e.,
\begin{equation}
     \yhat(\xx) \defeq \sigma(\alpha \cdot \tau(\vtheta^{\top} \xx + \phi) + \beta) =  
\begin{cases}
   \frac{1}{1+e^{-\beta}} \defeq a,& \textbf{if } \vtheta^{\top} \xx + \phi \leq 0 \\
    \frac{1}{1+e^{-\alpha - \beta}} \defeq b,              & \textbf{else}
\end{cases} 
 \end{equation}
This results in a classifier with the following properties, if $\vtheta^{\top} \xx + \phi \leq 0$, we have
\begin{align}
p(\YHAT = 0 \mid \X = \xx) &= a \\
p(\YHAT = 1 \mid \X = \xx) &= 1 - a
\end{align}
Otherwise, if $\vtheta^{\top} \xx + \phi > 0$, we have
\begin{align}
p(\YHAT = 0 \mid \X = \xx) &= b \\
p(\YHAT = 1 \mid \X = \xx) &= 1-b
\end{align}
By assumption, we have $\vtheta^{\top} \xx + \phi \neq  0$.
Now, observe that a binary $\delta$-discretized classifier can represent any distribution of the form 
\begin{align}
r(0) &= \begin{cases}
    \delta & \textbf{if } \vtheta^{\top}\xx + \phi > 0 \\
    1-\delta & \textbf{else}
\end{cases}\\
r(1) &= \begin{cases}
   1- \delta & \textbf{if } \vtheta^{\top}\xx + \phi > 0 \\
    \delta & \textbf{else}
\end{cases}
\end{align}
We show how to represent $r$ as a $\delta$-discretized binary log-linear model in four cases:
\begin{itemize}
    \item \textbf{Case 1}: $a > \frac{1}{2}$ and $b < \frac{1}{2}$. 
    In this case, we require a classifier that places probability $\delta$ on 0 if $\vtheta^{\top} \xx + \phi < 0$ and probability $1-\delta$ on 0 if $\vtheta^{\top} \xx + \phi < 0$.

    \item \textbf{Case 2}: $a < \frac{1}{2}$ and $b > \frac{1}{2}$. 
        In this case, we require a classifier that places probability $\delta$ on 0 if $\vtheta^{\top} \xx + \phi < 0$ and probability $1-\delta$ on 0 if $\vtheta^{\top} \xx + \phi > 0$.
    \item \textbf{Case 3}: $a, b > \frac{1}{2}$.
    In this case, we set $\vtheta = \mathbf{0}$ and $\phi > 0$. 
    
    \item \textbf{Case 4}: $a, b < \frac{1}{2}$.
    In this case, we set $\vtheta = \mathbf{0}$ and $\phi < 0$. 
\end{itemize}
This proves the result.

\end{proof}

\subsection{Accuracy-based Guardedness: the Balanced Case}
\label{app:acc-guardedness-ind}

\independence*

\begin{proof}
In the following proof, we use the notation $\X=\xx$ for the guarded variable $h(\X)=h(\xx)$ to avoid notational cutter.
Assume, by way of contradiction, that the $L_1$ independence gap (\cref{def:ind}), $ \sum_{y}  \left | \underset{\xx \sim p(\X \mid \Z = \bot)}{\mathbb{E}} p(\YHAT=y \mid \Z=\bot,\X = \xx) - \underset{\xx \sim p(\X \mid \Z = \top)}{\mathbb{E}} 
 p(\YHAT=y \mid \Z= \top,\X = \xx) \right | > 4 \varepsilon$.
Then, there exists a $y \in \yset$ such that 

\begin{equation}
    \left| \underset{h(\xx) \sim p(h(\X) \mid \Z=\bot)}{\mathbb{E}} p(\YHAT=y \mid \Z=\bot,\X = \xx)-\underset{\xx \sim p(\X \mid \Z=\top)}{\mathbb{E}} p(\YHAT=y \mid \Z=\top,\X = \xx) \right | > 2 \varepsilon.
\end{equation}

We will show that we can build a classifier $q^{\star} \in \Vfam$ that breaks the assumption $\Iacc(h(\X) \rightarrow \Z) < \varepsilon$. 
Next, we define the random variable $\ZHATq$ for convenience as
\begin{equation}
     \ZHATq(z) \defeq 
\begin{cases}
   1,& \textbf{if } z = \underset{z'}{\argmax}{\,\,q(z' \mid \xx)} \\
    0,              & \textbf{else}
\end{cases} 
\end{equation}
In words, $\ZHATq$ is a random variable that ranges over possible predictions, derived from the argmax, of the binary log-linear model $q$.
Now, consider the following two cases.
\begin{itemize}
\item \textbf{Case 1}: 
There exits a $y$ such that $\underset{\xx \sim p(\X \mid \Z=\bot)}{\mathbb{E}}  p(\YHAT=y \mid \Z=\bot,\X = \xx) -  \underset{\xx \sim p(\X \mid \Z)}{\mathbb{E}} p(\YHAT=y\mid \Z=\top,\X = \xx) >  2\varepsilon$. 
Let $\YHAT$ be defined as in \Cref{eq:binary-yhat}.
Next, consider a random variable $\ZHATr$ defined as follows
\begin{subequations}
\begin{align}
     p(\ZHATr = \bot \mid \YHAT = y) &\defeq 
\begin{cases}
   1,& \textbf{if}\,\, \YHAT = y \\
    0,              & \textbf{else}
\end{cases}  \\
p(\ZHATr = \top \mid \YHAT = y) &\defeq 
\begin{cases}
   1,& \textbf{if}\,\, \YHAT \neq y\\
    0,              & \textbf{else}
\end{cases} 
\end{align}
\end{subequations}
Now, note that we have
\begin{subequations}
\begin{align}
p(\ZHATr = \bot \mid\X = \xx) &= \sum_{y \in \yset} p(\ZHATr = \bot \mid \YHAT = y) p(\YHAT = y \mid\X = \xx) \\
&= p(\YHAT = y \mid\X = \xx)
\end{align}
\end{subequations}
and 
\begin{subequations}
\begin{align}
p(\ZHATr = \top \mid\X = \xx) &= \sum_{y \in \yset} p(\ZHATr = \top \mid \YHAT = y) p(\YHAT = y \mid\X = \xx) \\
&= p(\YHAT \neq y \mid\X = \xx)
\end{align}
\end{subequations}
We perform the algebra below where the step from \cref{eq:hard-step1} to \cref{eq:hard-step2} follows because of the fact that, despite the nuisance variable $\YHAT$, the decision boundary of $p(\ZHATr = \top \mid\X = \xx)$ is linear and, thus, there exists a binary log-linear model in $\Vfam$ which realizes it.
Now, consider the following steps
\begin{subequations}
 \begin{align}
\CAv&\left(\Z \mid \X\right)\defeq \sup_{q \in \Vfam} \E{(\xx,z) \sim p(\Z, \X)} \ell(\q, \xx, z) \\
 &= \sup_{q \in \Vfam}  \underset{\xx \sim p(\X \mid \Z)}{\mathbb{E}} \E{z \sim p(\Z)} \ell(\q, \xx, z) \\
&= \sup_{q \in \Vfam} \underset{\xx \sim p(\X \mid \Z =\bot)}{\mathbb{E}} p(\ZHATq = \bot \mid \Z = \bot,\X = \xx )p(\Z = \bot) \label{eq:hard-step1}\\
& \quad\quad\quad\quad  + \underset{\xx \sim p(\X \mid \Z =\top)}{\mathbb{E}} p(\ZHATq = \top \mid \Z = \top,\X = \xx )p(\Z = \top) \nonumber \\
     &\geq  \underset{\xx \sim p(\X \mid \Z = \bot)}{\mathbb{E}} p( \ZHATr = \bot \mid \Z = \bot,\X = \xx)p(\Z = \bot) \label{eq:hard-step2} \\
     &\quad\quad\quad\quad +  \underset{\xx \sim p(\X \mid \Z = \top)}{\mathbb{E}} p(\ZHATr = \top \mid \Z = \top,\X = \xx )p(\Z = \top) \nonumber  \\
    &=  \underset{\xx \sim p(\X \mid \Z = \bot)}{\mathbb{E}} p( \YHAT = y \mid \Z = \bot,\X = \xx )p(\Z = \bot) \\ 
    & \quad\quad\quad\quad + \underset{\xx \sim p(\X \mid \Z = \top)}{\mathbb{E}} p\left( \YHAT \neq y \mid \Z = \top,\X = \xx \right)p(\Z = \top) \\
        &= \underset{\xx \sim p(\X \mid \Z = \bot)}{\mathbb{E}}p(\YHAT=y \mid \Z=\bot,\X = \xx)p(\Z=\bot) \\
      &  \quad\quad\quad\quad +  \underset{\xx \sim p(\X \mid \Z = \top)}{\mathbb{E}} (1-p(\YHAT=y\mid \Z=\top,\X = \xx))p(\Z=\top) \\
        &=  \underset{\xx \sim p(\X \mid \Z=\bot)}{\mathbb{E}}\frac{1}{2}p(\YHAT=y \mid \Z=\bot,\X = \xx) \\ 
       & \quad\quad\quad\quad + \underset{\xx \sim p(\X \mid \Z = \top)}{\mathbb{E}}\frac{1}{2}(1-p(\YHAT=y\mid \Z=\top,\X = \xx)) \\
    &=  \underset{\xx \sim p(\X \mid \Z=\bot)}{\mathbb{E}}\frac{1}{2}p(\YHAT=y \mid \Z=\bot,\X = \xx) \\ 
    & \quad\quad\quad\quad +  \underset{\xx \sim p(\X \mid \Z=\top)}{\mathbb{E}} \frac{1}{2}-\frac{1}{2}p(\YHAT=y\mid \Z=\top,\X = \xx) \\
        &=  \frac{1}{2}\underbrace{\underset{\xx \sim p(\X \mid \Z=\bot)}{\mathbb{E}}p(\YHAT=y \mid \Z=\bot,\X = \xx) -\!\!\!\! \underset{\xx \sim p(\X \mid \Z=\top)}{\mathbb{E}} p(\YHAT=y\mid \Z=\top,\X = \xx)}_{> 2\varepsilon \text{ by assumption}}  + \frac{1}{2} \\
            &> \frac{1}{2}\left(2 \varepsilon\right)+ \frac{1}{2} \\
            &= \frac{1}{2} + \varepsilon
\end{align}
\end{subequations}

\item \textbf{Case 2}: There exits a $y$ such that 

\begin{equation}
    \underset{\xx \sim p(\X \mid \Z =\top)}{\mathbb{E}} p(\YHAT=y \mid \Z=\top,\X = \xx) -   \underset{\xx \sim p(\X \mid \Z =\bot)}{\mathbb{E}} p(\YHAT=y\mid \Z=\bot,\X = \xx) >  2\varepsilon 
\end{equation}

Let $\YHAT$ be defined as in \Cref{eq:binary-yhat}.
Next, consider a random variable defined as follows
\begin{subequations}
\begin{align}
     p(\ZHATr = \bot \mid \YHAT = y) &\defeq 
\begin{cases}
   1,& \textbf{if}\,\, \YHAT \neq y \\
    0,              & \textbf{else}
\end{cases}  \\
p(\ZHATr = \top \mid \YHAT = y) &\defeq 
\begin{cases}
   1,& \textbf{if}\,\, \YHAT = y\\
    0,              & \textbf{else}
\end{cases} 
\end{align}
\end{subequations}

Now, note that we have
\begin{subequations}
\begin{align}
p(\ZHATr = \bot \mid\X = \xx) &= \sum_{y \in \yset} p(\ZHATr = \bot \mid \YHAT = y) p(\YHAT = y \mid\X = \xx) \\
&= p(\YHAT \neq y \mid\X = \xx)
\end{align}
\end{subequations}
and 
\begin{subequations}
\begin{align}
p(\ZHATr = \top \mid\X = \xx) &= \sum_{y \in \yset} p(\ZHATr = \top \mid \YHAT = y) p(\YHAT = y \mid\X = \xx) \\
&= p(\YHAT = y \mid\X = \xx)
\end{align}
\end{subequations}
We proceed by algebraic manipulation
\begin{subequations}
 \begin{align}
\CAv&\left(\Z \mid \X\right)\defeq \sup_{q \in \Vfam} \E{(\xx,z) \sim p(\Z, \X)} \ell(\q, \xx, z) \\
 &= \sup_{q \in \Vfam}  \underset{\xx \sim p(\X \mid \Z)}{\mathbb{E}} \E{z \sim p(\Z)} \ell(\q, \xx, z) \\
&= \sup_{q \in \Vfam} \underset{\xx \sim p(\X \mid \Z =\bot)}{\mathbb{E}} p(\ZHATq = \bot \mid \Z = \bot,\X = \xx )p(\Z = \bot) \label{eq:hard-step1}\\
& \quad\quad\quad\quad  + \underset{\xx \sim p(\X \mid \Z =\top)}{\mathbb{E}} p(\ZHATq = \top \mid \Z = \top,\X = \xx )p(\Z = \top) \nonumber \\
     &\geq  \underset{\xx \sim p(\X \mid \Z = \bot)}{\mathbb{E}} p( \ZHATr = \bot \mid \Z = \bot,\X = \xx)p(\Z = \bot) \label{eq:hard-step2} \\
     &\quad\quad\quad\quad +  \underset{\xx \sim p(\X \mid \Z = \top)}{\mathbb{E}} p(\ZHATr = \top \mid \Z = \top,\X = \xx )p(\Z = \top) \nonumber  \\
    &=  \underset{\xx \sim p(\X \mid \Z = \bot)}{\mathbb{E}} p( \YHAT \neq y \mid \Z = \bot,\X = \xx )p(\Z = \bot) \\ 
    & \quad\quad\quad\quad + \underset{\xx \sim p(\X \mid \Z = \top)}{\mathbb{E}} p\left( \YHAT = y \mid \Z = \top,\X = \xx \right)p(\Z = \top) \\
            &= \underset{\xx \sim p(\X \mid \Z = \bot)}{\mathbb{E}}(1-p(\YHAT=y \mid \Z=\bot,\X = \xx))p(\Z=\bot) \\
      &  \quad\quad\quad\quad +  \underset{\xx \sim p(\X \mid \Z = \top)}{\mathbb{E}} p(\YHAT=y\mid \Z=\top,\X = \xx)p(\Z=\top) \\
        &=  \underset{\xx \sim p(\X \mid \Z=\top)}{\mathbb{E}}\frac{1}{2}p(\YHAT=y \mid \Z=\top,\X = \xx) \\ 
       & \quad\quad\quad\quad + \underset{\xx \sim p(\X \mid \Z = \bot)}{\mathbb{E}}\frac{1}{2}(1-p(\YHAT=y\mid \Z=\bot,\X = \xx)) \\
    &=  \underset{\xx \sim p(\X \mid \Z=\top)}{\mathbb{E}}\frac{1}{2}p(\YHAT=y \mid \Z=\top,\X = \xx) \\ 
    & \quad\quad\quad\quad +  \underset{\xx \sim p(\X \mid \Z=\bot)}{\mathbb{E}} \frac{1}{2}-\frac{1}{2}p(\YHAT=y\mid \Z=\bot,\X = \xx) \\
        &=  \frac{1}{2}\underbrace{\underset{\xx \sim p(\X \mid \Z=\top)}{\mathbb{E}}p(\YHAT=y \mid \Z=\top,\X = \xx) -\!\!\!\! \underset{\xx \sim p(\X \mid \Z=\bot)}{\mathbb{E}} p(\YHAT=y\mid \Z=\bot,\X = \xx)}_{> 2\varepsilon \text{ by assumption}}  + \frac{1}{2} \\
            &> \frac{1}{2}\left(2 \varepsilon\right)+ \frac{1}{2} \\
            &= \frac{1}{2} + \varepsilon
\end{align}
\end{subequations}

\end{itemize}
In both cases, we have $\CAv\left(\Z \mid\X = \xx \right) > \frac{1}{2} + \varepsilon$.
Thus, $\mathbb{E}_{\xx \sim p(\X)} \CAv\left(\Z \mid\X = \xx \right) = \CAv\left(\Z \mid h(\X) \right) \geq \frac{1}{2} + \varepsilon$.
Note that the distribution $p(\Z, \X)$ is globally balanced, we have $\CAv\left(\Z\right)=\frac{1}{2}$. 
Thus, 
\begin{subequations}
\begin{align}
\Iacc(h(\X) \rightarrow \Z) 
&=\CAv\left(\Z\right) -  \CAv\left(\Z \mid h(\X)\right) \\
&=\CAv\left(\Z \mid h(\X)\right) - \frac{1}{2} \\
&> \frac{1}{2} + \varepsilon- \frac{1}{2} \\
&= \varepsilon
\end{align}
\end{subequations}
However, this contradicts the assumption that $\Iacc(h(\X) \rightarrow \Z) < \varepsilon$.
This completes the proof.
\end{proof}

\clearpage
\newpage

\subsection{Experimental Setting}
\label{app:experiments}

In this appendix, we give additional information necessary to replicate our experiments (\cref{sec:experiments}).

\paragraph{Data.} We use the same train--dev--test split of the biographies dataset used by \citet{ravfogel-etal-2020-null}, resulting in training, evaluation and test sets of sizes 255,710, 39,369, and 98,344, respectively. We reduce the dimensionality of the representations to 256 using PCA. The dataset is composed of short biographies, annotated with both gender and profession. We randomly sampled 15 pairs of professions from the dataset: (\emph{professor}, \emph{attorney}),  (\emph{journalist}, \emph{surgeon}),  (\emph{physician}, \emph{nurse}),  (\emph{professor}, \emph{physician}),  (\emph{psychologist}, \emph{teacher}),  (\emph{attorney}, \emph{teacher}),  (\emph{physician}, \emph{journalist}),  (\emph{professor}, \emph{dentist}),  (\emph{teacher}, \emph{surgeon}),  (\emph{psychologist}, \emph{surgeon}),  (\emph{photographer}, \emph{surgeon}),  (\emph{attorney}, \emph{psychologist}),  (\emph{physician}, \emph{teacher}),  (\emph{professor}, \emph{teacher}),  (\emph{professor}, \emph{psychologist})

\paragraph{Optimization.} We run RLACE \citep{ravfogel2022linear} with a simple SGD optimization, with a learning rate of $0.005$, a weight decay of $1e^{-5}$ and a momentum of $0.9$, chosen by experimenting with the development set. We use a batch size of 128. The algorithm is based on an adversarial game between a predictor that aims to predict gender, and an orthogonal projection matrix adversary that aims to prevent gender classification. We choose the adversary which yielded \emph{highest} classification loss. All training is done on a single NVIDIA GeForce GTX 1080 Ti GPU.

\paragraph{Estimating $\Vfam$-information.} After running RLACE, we get an approximately linearly-guarded representation by projecting $\xxn \gets P\xxn$, where $P$ is the orthogonal projection matrix returned by RLACE. We validate guardedness by training log-linear models over the projected representations; they achieve accuracy less than $2\%$ above the majority accuracy. Then, to estimate $\Ivyzadv$, we fit a simple neural network of the form of a composition of two log-linear models. The inner model either has a single hidden neuron with a logistic activation (in the binary experiment), or $K=2,4,8,16,32,64$ hidden neurons with softmax activations, in the multiclass experiment (\cref{sec:multi-experiment}). The networks are trained end to end to recover binary gender for $25000$ batches of size $2048$. Optimization is done with Adam with the default parameters. We use the loss of the second log-linear model to estimate $\Ivyzadv$, according to \cref{def:information-theoretic-guardedness}.

\end{document}